\newif\ifarXiv         
\newif\ifjournal       
\journalfalse \arXivtrue 

\documentclass[11pt]{article}

\usepackage[small,compact]{titlesec}
\usepackage{amsmath,amssymb,amsfonts,mathabx,setspace,amsbsy}
\usepackage{mathtools}
\usepackage{mathrsfs}
\usepackage{cases}
\usepackage{booktabs,bigstrut}
\usepackage{blkarray}
\usepackage{graphicx,caption,epsfig}
\usepackage{epsfig,wrapfig}
\usepackage[dvipsnames,svgnames]{xcolor}
\usepackage[normalem]{ulem}

\usepackage[
  backend=bibtex,   
  style=numeric,
  doi=false,
  url=false,
  isbn=false,
  eprint=false,
  maxbibnames=99,
  minbibnames=99,
]{biblatex}

\AtEveryBibitem{%
  \clearname{editor}%
  \clearlist{editor}%
  \clearfield{editora}%
  \clearfield{editorb}%
  \clearfield{editorc}%
}

\AtEveryBibitem{%
  \clearlist{publisher}%
  \clearlist{location}%
}

\AtEveryBibitem{%
  \clearfield{langid}%
  \clearfield{language}%
}

\usepackage{authblk}
\usepackage{graphicx,epsfig,wrapfig,caption,epsfig,subcaption,sidecap}
\usepackage{url,color,verbatim}
\usepackage{algorithmicx}
\usepackage[noend]{algpseudocode}
\usepackage{bbm}
\usepackage{enumerate}
\usepackage{enumitem}
\usepackage{booktabs} 

\usepackage[scaled]{helvet}
\usepackage{soul}
\usepackage[T1]{fontenc}
\usepackage{empheq}

\usepackage[ruled,boxed]{algorithm}

\usepackage{tikz}
\usetikzlibrary{positioning}
\usetikzlibrary{arrows,shapes}
\usetikzlibrary{decorations.markings}
\usetikzlibrary{arrows.meta}
\usetikzlibrary{shapes.multipart}
\usetikzlibrary{shapes.geometric}

\usepackage{pgfplots}
\pgfplotsset{compat=1.14}
\usetikzlibrary{calc,shapes.callouts}
\usetikzlibrary{arrows,shapes,trees,positioning}  

\usepackage[customcolors]{hf-tikz}

\usepackage[bookmarks=true, bookmarksnumbered=true, colorlinks=true,   pdfstartview=FitV,
linkcolor=blue, citecolor=blue, urlcolor=blue]{hyperref}

\usepackage{authblk}
\usepackage{mathtools}

\usepackage[english]{babel} 
\usepackage{multirow}

\let\OLDthebibliography\thebibliography
\renewcommand\thebibliography[1]{
  \OLDthebibliography{#1}
  \setlength{\parskip}{0pt}
  \setlength{\itemsep}{3pt}
}

\usepackage[
  textwidth=1in,    
  textsize=tiny,  
  linecolor=blue,      
  bordercolor=blue     
]{todonotes}

\usepackage[
  left=0.8in,
  right=0.8in,  
  top=1in,
  bottom=1in
]{geometry}
\def\ba{\mathbf{a}}
\def\bb{\mathbf{b}}
\def\bc{\mathbf{c}}

\def\be{\mathbf{e}}

\def\bm{\mathbf{m}}

\def\br{\mathbf{r}}

\def\bu{\mathbf{u}}
\def\bv{\mathbf{v}}

\def\bx{\mathbf{x}}
\def\by{\mathbf{y}}

\def\bkappa{\mathbf{\kappa}}

\def\bA{\mathbf{A}}
\def\bB{\mathbf{B}}

\def\bS{\mathbf{S}}

\def\bX{\mathbf{X}}

\def\bZ{\mathbf{Z}}

\def\bb0{\mathbb{0}}
\def\bb1{\mathbb{1}}

\def\calB{\mathcal{B}}

\def\calH{\mathcal{H}}

\def\calK{\mathcal{K}}
\def\calL{\mathcal{L}}

\def\calO{\mathcal{O}}

\def\calQ{\mathcal{Q}}

\def\sA{\mathscr{A}}

\def\E{\mathbb{E}}

\def\R{{\mathbb R}}

\def\bf1{\textbf{1}}

\def\mH{\mathcal{H}}
\def\mL{\mathcal{L}}

\def\mN{\mathcal{N}}

\def\Tr{{\rm Tr}}
\def\Id{{\rm Id}}

\def\Diag{{\rm Diag}}
\def\rank{{\rm rank}}

\newcommand{\IK}{\Phi}

\renewcommand{\top}{{\rm T}}

\def\bkappa{\boldsymbol{\kappa}}

\newcommand{\bPhi}{\boldsymbol{\Phi}}

\newcommand{\bKa}{\boldsymbol{\mathcal{K}}}
\newcommand{\bEta}{\boldsymbol{\eta}}

\newcommand{\norm}[1]{\left\|#1\right\|}        
\newcommand{\abs}[1]{\left|#1\right|}           
\newcommand{\innerp}[1]{\langle{#1}\rangle}     

\newcommand{\argmin}{\operatornamewithlimits{argmin}} 

\newtheorem{theorem}{Theorem}

\newtheorem{assumption}[theorem]{Assumption}

\newtheorem{definition}[theorem]{Definition}
\newtheorem{example}[theorem]{Example}
\newtheorem{lemma}[theorem]{Lemma}
\newtheorem{proposition}[theorem]{Proposition}
\newtheorem{remark}[theorem]{Remark}

\newenvironment{proof}[1][Proof]{\noindent\textbf{#1.} }{\ \rule{0.5em}{0.5em}}

\numberwithin{equation}{section}
\numberwithin{theorem}{section}

\newcommand{\fcommentout}[1]{}

\allowdisplaybreaks[4] 

\title{Learning Multi-type heterogeneous interacting particle systems}
\author[1]{Quanjun Lang\thanks{quanjun.lang@duke.edu; co-first authors}}
\author[2]{Xiong Wang\thanks{xiongwang@ualberta.ca; co-first authors}}
\author[3]{Fei Lu\thanks{feilu@math.jhu.edu}}
\author[3,4]{Mauro Maggioni\thanks{mauromaggionijhu@icloud.com}}
\affil[1]{Department of Mathematics, Duke University, Durham, USA. }
\affil[2]{School of Mathematics, Sun Yat-sen University, Guangzhou, China. }
\affil[3]{Department of Mathematics, Johns Hopkins University, Baltimore, USA. }
\affil[4]{Department of Applied Mathematics and Statistics, Johns Hopkins University, Baltimore, USA. }
\date{}
\begin{document}
\maketitle

\begin{abstract}
We propose a framework for the joint inference of network topology, multi-type interaction kernels, and latent type assignments in heterogeneous interacting particle systems from multi-trajectory data. This learning task is a challenging non-convex mixed-integer optimization problem, which we address through a novel three-stage approach. First, we leverage shared structure across agent interactions to recover a low-rank embedding of the system parameters via matrix sensing. Second, we identify discrete interaction types by clustering within the learned embedding. Third, we recover the network weight matrix and kernel coefficients through matrix factorization and a post-processing refinement. We provide theoretical guarantees with estimation error bounds under a Restricted Isometry Property (RIP) assumption and establish conditions for the exact recovery of interaction types based on cluster separability. Numerical experiments on synthetic datasets, including heterogeneous predator-prey systems, demonstrate that our method yields accurate reconstruction of the underlying dynamics and is robust to noise.
\end{abstract}

\setcounter{tocdepth}{1}
\tableofcontents

\section{Introduction}\label{sec:Intro}
Interacting particle systems (IPS) are widely used to model the collective dynamics of agents interacting through pairwise forces. These models arise in diverse fields such as biology, where they describe cell-cell signaling and collective migration \cite{camleyPhysicalModelsCollective2017, deutschBIOLGCACellularAutomaton2021}; ecology, where they are used to study predator-prey interactions \cite{satulovskyStochasticLatticeGas1994, schinaziPredatorpreyHostparasiteSpatial1997}; physics, where they model granular and molecular systems \cite{baumgartenGeneralConstitutiveModel2019, bellParticlebasedSimulationGranular2005, kamrinAdvancesModelingDense2024}; and the social sciences, where they capture opinion dynamics and crowd behavior \cite{hegselmannOPINIONDYNAMICSBOUNDED}, just to name a few.

In many of these systems, interactions are heterogeneous: different types of agents may respond differently to one another, and the interaction strength may vary across pairs. Such heterogeneity may arise from agent types such as species or roles, from asymmetric relationships such as leaders versus followers, or from context-dependent behavior. To capture these effects, IPS models must be extended to incorporate multiple interaction kernels, selected according to a discrete type matrix that encodes agent-type or pair-type relations.

However, although substantial progress has been made in learning IPSs from trajectory data, prior work has explored only homogeneous systems or special cases of heterogeneous systems. 
Estimation of the interaction kernels for homogeneous IPSs has been studied in~\cite{LZTM19pnas,LMT21,FRT21} for differential equations and in~\cite{messenger2022learning,Messenger2022Lair,LangLu22} for mean-field equations. Inference for heterogeneous systems has been studied in \cite{LMT21_JMLR} for kernel estimation and in \cite{LWLM2024} for joint estimation of the network and interaction kernel with known types. 

In this study, we address the challenging problem of learning heterogeneous systems on networks with multiple interaction types. We consider the general setting where the network topology, the interaction kernels, and the specific type assignments governing pairwise interactions are all unknown. Our goal is to recover these latent structures simultaneously from observed trajectories. This leads to a challenging non-convex mixed-integer optimization problem, which we tackle through a novel three-stage approach involving low-rank matrix sensing, clustering and matrix factorization.


\subsection{Problem statement}
Consider a heterogeneous system of $N$ agents on a network interacting according to the governing equation 
\begin{equation}\label{Eq:ips_K_type}
d {X}^i_t =  
\sum_{j\neq i} a_{ij}\IK_{\kappa_{ij}}(X^j_t-X^i_t) dt +  \sigma d {W}^i_t, \quad i \in [N]\,,
\end{equation}
where $X_t^i\in\R^d$ is the state of the $i$-th agent at time $t$, $[N]:= \{1,\cdots,N\}$, and we write $\sum_{j\neq i}$ for $\sum_{j\in[N]\setminus\{i\}}$. We assume that the \emph{graph matrix} $\ba \in [0, 1]^{N \times N}$ is row-normalized, and determines the network of interactions.
There are $Q$ types of interaction kernels $\{\Phi_1, \dots, \Phi_Q\}$, and the integer-valued entry $\kappa_{ij}\in [Q]$ of the \emph{type matrix} $\bkappa= (\kappa_{ij})_{1\leq i,j\leq N}\in [Q]^{N\times N}$ specifies the type of interaction kernel between agent $i$ and agent $j$. The stochastic force term $\sigma dW_t^i$ with $\{W_t^i\}$ is a standard Brownian motion on $\R^d$ representing the environmental noise; the system is deterministic when $\sigma=0$. In this work we assume that there is a given set of basis functions $\{\psi_k\}$ such that each interaction kernel $\Phi_q\in \calH = \mathrm{span}\{\psi_k\}_{k=1}^K$, so that $\Phi_q$ is a linear combination of the $\psi_k$'s with coefficients $\bc^{(q)} = [c_1^{(q)}, \dots, c_K^{(q)}]^\top \in \R^K$, and therefore the collection of the interaction kernels $(\Phi_1, \dots, \Phi_Q)$ corresponds to a  \emph{coefficient matrix} $\bc$: 
\begin{equation}\label{Def:bc-Mat}
\Phi_q = \sum_{k=1}^{K} c_k^{(q)}\psi_k\,, \quad q\in[Q], \quad  \quad  \bc=\left[\bc^{(1)},\bc^{(2)},\cdots,\bc^{(Q)}\right] \in \R^{K\times Q}\,.
\end{equation} 
In this study, we assume that the basis functions $\{\psi_k\}$ are known, so that the system \eqref{Eq:ips_K_type} depends on $\ba, \bc$ and $\bkappa$.

Our goal is to estimate the unknown parameter triplet $(\mathbf{a}^*, \bc^*, \bkappa^*)$ from observation data consisting of multiple trajectories of the system \eqref{Eq:ips_K_type}. The observed data are in the following form: 
\begin{equation}\label{eq:data}
	\textbf{Data: } \quad \{ \bX_{t_0:t_L}^m, \dot{\bX}_{t_0:t_L}^m\}_{m=1}^M\,,
\end{equation}
where $t_0:t_L$ denotes $(t_0,t_1,t_2,\ldots,t_L)$ with $t_l= l\Delta t$. Here, $\bX_t$ and $\dot{\bX}_{t_l}$ denote the state of the system and the finite-difference approximation of the velocity at time $t_l$:  
\begin{equation}\label{eq:stack_X}
\bX_t := [X^1_t, \cdots, X^N_t]^\top \in \R^{dN}, \quad \dot{\bX}_{t_l}:=(\Delta t)^{-1}(\bX_{t_{l+1}} - \bX_{t_l})\in \R^{dN}. 
\end{equation}
The superscript $m$ indexes independent trajectory samples, with initial conditions $\{\bX^m_{t_0}\}$ sampled independently from $\mu^{\otimes N}$, where $\mu$ is a probability distribution on $\R^d$. 
The regime of interest here is when the number of agents $N$ is moderately large, the total number of interaction kernel types $Q$ remains relatively small compared to $N$, and the hypothesis-space dimension $K$ may be large, potentially scaling with the sample size $M$ of observed trajectories in a nonparametric setting. 

\begin{example}[Heterogeneous interaction kernel systems]\label{Exa:Heter_interaction}
	Consider a one-dimensional system with $N=5$ agents and $Q=3$ types of interactions, namely $\bPhi = \{\Phi_1, \Phi_2, \Phi_3\}$.
	We assume the type matrix is 
	\begin{align}\label{Eg:agentAssign}
	\bkappa^*=\left[\begin{array}{ccccc}
    *   & 2  & 3 & 1 & 1  \\
    1   & *  & 1 & 3 & 2 \\
    2   & 1  & * & 2 & 3\\
    1   & 3  & 3 & * & 3\\
    2   & 1  & 1 & 1 & *
    \end{array}\right]\in\{1, 2, 3\}^{N \times N}\,.
	\end{align}
	where the diagonal elements are irrelevant since we exclude self-interactions. Denote $\br^{i,j}_t:=X^j_t-X^i_t$. The governing equation \eqref{Eq:ips_K_type} can be written explicitly as
	\begin{align*}
		\begin{cases}
			dX_t^1 = \big[a_{12}\Phi_{2}(\br^{1,2}_t)+ a_{13}\Phi_{3}(\br^{1,3}_t)+ a_{14}\Phi_{1}(\br^{1,4}_t)+ a_{15}\Phi_{1}(\br^{1,5}_t)\big] dt + \sigma dW_t^1\,, \\
			dX_t^2 = \big[a_{21}\Phi_{1}(\br^{2,1}_t)+ a_{23}\Phi_{1}(\br^{2,3}_t)+ a_{24}\Phi_{3}(\br^{2,4}_t)+ a_{25}\Phi_{2}(\br^{2,5}_t)\big] dt + \sigma dW_t^2\,, \\
			dX_t^3 = \big[a_{31}\Phi_{2}(\br^{3,1}_t)+ a_{32}\Phi_{1}(\br^{3,2}_t)+ a_{34}\Phi_{2}(\br^{3,4}_t)+ a_{35}\Phi_{3}(\br^{3,5}_t)\big] dt + \sigma dW_t^3\,, \\
			dX_t^4 = \big[a_{41}\Phi_{1}(\br^{4,1}_t)+ a_{42}\Phi_{3}(\br^{4,2}_t)+ a_{43}\Phi_{3}(\br^{4,3}_t)+ a_{45}\Phi_{3}(\br^{4,5}_t)\big] dt + \sigma dW_t^4\,, \\
			dX_t^5 = \big[a_{51}\Phi_{2}(\br^{5,1}_t)+ a_{52}\Phi_{1}(\br^{5,2}_t)+ a_{53}\Phi_{1}(\br^{5,3}_t)+ a_{54}\Phi_{1}(\br^{5,4}_t)\big] dt + \sigma dW_t^5\,,
		\end{cases}
	\end{align*}
	where $\ba$ represents the network structure connecting the agents.

\end{example}

\begin{remark}[Complete networks]\label{rmk:complete_network}
If $\ba = \frac{1}{N-1}\left(\bf1_N \bf1_N^\top - \mathbf{I}_N\right)$, that is, $a_{ij} = \frac{1}{N-1}$ for $i \neq j$ and $a_{ii} = 0$, the system corresponds to an IPS on a \emph{complete network} with $Q$ distinct kernels, where the network interaction magnitude is identical for all particle pairs. 
If $Q = 1$, we recover the IPS on a graph with a single interaction kernel, as studied in \cite{LWLM2024}. 
The graph can be interpreted as encoding interaction magnitudes determined by particle-specific parameters. For example, if particles have different masses, the network matrix can be written as $\ba = \bm \bm^\top$, where $\bm \in \mathbb{R}^N$ contains the masses of the particles. The different types of interactions are natural in practice. 

\end{remark}

\begin{remark}[Kernel heterogeneity]
	Interaction kernels in IPS can be assigned according to different structural rules. 
	In \emph{species-based} heterogeneity, the kernel depends only on the species of the interacting agents. 
	A more general \emph{ordered-pair} structure assigns a distinct kernel to each ordered type pair $(i,j)$ \cite[Sec.~4]{LZTM19pnas}, \cite{LMT21_JMLR}. 
	Alternatively, the kernel may depend solely on the \emph{receiving} particle (\emph{row-wise} dependence) or solely on the \emph{forcing} particle (\emph{column-wise} dependence) \cite[Sec.~4.3]{LWLM2024}. 
	Our work considers the most general setting, with an arbitrary, unknown type matrix $\bkappa$ and unknown $Q$, allowing fully heterogeneous interaction patterns without prior structural constraints.
\end{remark}

\subsection{Overview of the learning framework}\label{sec:overview}
We consider the commonly used estimator obtained by minimizing the loss function of the empirical mean-squared error of the predicted velocities of the particles

\begin{equation} \label{Eq:Estimator_Loss}
\begin{split}
	(\widehat{\ba}, \widehat{\bc}, \widehat{\bkappa}) 
	&= \argmin{} \calL(\ba, \bc,\bkappa), \quad \text{with }
	\calL(\ba, \bc,\bkappa)
	= \frac{1}{N} \sum_{i=1}^N \calL_i(\ba_i, \bc,\bkappa), \quad \text{where }
	\\
 	\mL_i(\ba_{i}, \bc, \bkappa_{i}) 
	&= \frac{1}{LM}\sum_{l=1}^L\sum_{m=1}^M
	\norm{
		\dot X^{i, m}_{t_l}
		- \sum_{j\neq i} a_{ij}
		\sum_{k=1}^{K} c_k^{(\kappa_{ij})} \psi_k(X^{j,m}_{t_l} - X^{i,m}_{t_l})
	}_{\R^{dN}}^2.
\end{split}
\end{equation}

A central challenge is that \eqref{Eq:Estimator_Loss} is a non-convex, mixed-integer optimization problem. The loss is non-convex because the drift term  $\sum_{j\neq i} a_{ij} \sum_{k=1}^{K} c_k^{(\kappa_{ij})}\psi_k (X^j_t-X^i_t) $ depends nonlinearly on the parameters $(\mathbf{a}, \bc, \bkappa)$, while the integer-valued type matrix $\bkappa$ introduces a combinatorial structure that makes direct optimization NP-hard.  As a result, classical linear regression-based methods, either parametric or nonparametric, which are extensively studied in statistics and machine learning (e.g.\ \cite{CuckerSmale02, CZ07book, LZTM19pnas, LMT21_JMLR, Messenger2022Lair}), are typically not applicable here.  Addressing \eqref{Eq:Estimator_Loss} therefore demands specialized non-convex optimization methods that explicitly handle the integer-valued type function.

To tackle the joint estimation problem~\eqref{Eq:Estimator_Loss}, we propose a three-stage approach that separates the estimation of continuous parameters ($\ba, \bc$) from the discrete parameters $\bkappa$. We provide only a high-level overview of the three-stage framework here; the formal definitions of the loss functions, the matrix sensing problem, and the clustering procedure appear in Section~\ref{Sec:Framework}, and algorithmic details are given in Section~\ref{Sec:Alg}. 
The key idea is to reparametrize the interaction structure of each agent $i$ using the matrix variable
\begin{align}
	\bZ_i = \Diag(\ba_i)\bKa_i \bc^\top \in \mathbb{R}^{N \times K}\,, \quad i\in [N], 
	\label{e:bZi}
\end{align}
which encodes how agent $i$ interacts with the rest of the system. Here $\bKa_i:= \bigl( \mathbbm{1}_{\{\kappa_{ij} = q\}} \bigr) \in \{0,1\}^{(N-1) \times Q}$ is the \emph{assignment matrix}; see Example \ref{Example:Type_Mat} and Section \ref{sec:assignmentMat} for further details. With this matrix, the loss function $\calL_i(\ba_i, \bc, \bkappa)$ in \eqref{Eq:Estimator_Loss} can be rewritten as
\begin{align}\label{Def:LossFunc_MatS}
	\calL_i(\bZ_i)&=\frac{1}{LM}\sum_{l,m=1}^{L,M} \norm{\dot{X}_{t_l}^{i, m} -  \innerp{\bB(\bX_{t_l}^m)_i, \bZ_i}_F}^2_{\R^d} \,,
\end{align}
where $\{\bB_i(\bX_{t_l}^m)\}_{m=1}^M$ are sensing matrices related to the trajectory observations and the basis functions, to be defined in \eqref{Def:basis_array}.
The three stages are the following:
\begin{itemize}
\item Stage 1: matrix sensing. We estimate the collection $\{ \bZ_i \}_{i=1}^N$ by solving the matrix sensing problem \eqref{Def:LossFunc_MatS} using the Alternating Least Square (ALS) algorithm, leveraging the fact that each $\bZ_i$ is low-rank and implicitly shares the same global kernel structure through $\bc$. This stage recovers the hidden low-dimensional structure in the dynamics from noisy trajectory data.

\item Stage 2: clustering. We estimate the type matrix $\widehat{\bkappa}$ by applying $K$-means clustering to the normalized nonzero rows of $\{\widehat{\bZ}_i\}_{i=1}^N$, leveraging the geometric separation between the rows of $\bc$, which represent the interactions. If the number of types $Q$ is unknown, it is determined via a separation-based criterion that compares intra- and inter-cluster angles. 

\item Stage 3: matrix decomposition and post-processing. 
We factorize the estimated arrays $\{\widehat{\bZ}_i\}_{i=1}^N$ to obtain the estimated graph matrix $\widehat{\ba}$ and the estimated kernel coefficient matrix $\widehat{\bc}$ by solving a linear system under the row-normalization constraint on $\widehat{\ba}$; this has a negligible computational cost. 
A final, post-processing ALS refinement is performed to enhance robustness to noise. 
\end{itemize}
The resulting estimates $(\widehat{\ba}, \widehat{\bc}, \widehat{\bkappa})$ provide the recovered network topology, interaction kernels, and type assignments governing the dynamics.

We also establish recovery guarantees for the proposed framework. In the first stage, under a Restricted Isometry Property (RIP) assumption on the sensing operator, we prove that the solution to matrix sensing problems \eqref{Def:LossFunc_MatS} achieves near-optimal accuracy in the Frobenius norm, with bounds that scale optimally with the sample size and noise level. In the second stage, we show that when the rows of $\bZ$ corresponding to different interaction types are geometrically well-separated, $K$-means clustering exactly recovers the type assignment $\widehat{\bkappa}$, even when $Q$ is unknown. The details are presented in Section~\ref{Sec:Theory}.

\medskip
\noindent\textbf{Main contributions.}
This work makes the following contributions: 
\begin{enumerate}
    \item \emph{A novel three-stage framework solving the non-convex mixed-integer optimization problem.} We introduce a formulation that reduces the non-convex mixed-integer optimization problem to a low-rank matrix sensing problem followed by geometric clustering, decoupling the continuous and discrete aspects of the mixed-integer optimization, providing a scalable solution. 
    \item \emph{Recovery guarantees.} We prove that, under a Restricted Isometry Property (RIP) assumption on the sensing operator, the ALS estimator of $\{\bZ_i\}_{i = 1}^N$ achieves optimal rates in both Frobenius norms, which in turn ensures accurate graph and kernel recovery. We also establish sufficient conditions, expressed in terms of inter- and intra-cluster angular separation and minimal row norms, under which $K$-means clustering on the estimated kernel-graph embedding $\{\bZ_i\}_{i = 1}^N$ exactly recovers $\bkappa$, even when the number of interaction kernel type $Q$ is unknown.
    \item \emph{Scalable and robust algorithms.} We design algorithms that scale to moderately large $N$ and $K$, include a post-processing ALS refinement to improve noise robustness, and apply equally to sparse or complete networks. 
     We demonstrate these algorithms on synthetic data, including heterogeneous predator-prey systems, and illustrate sharp phase transitions as a function of separation parameters and sample size.
\end{enumerate}

\subsection{Related work}\label{sec:related}

Learning structured interactions in heterogeneous interacting particle systems (IPS) combines inverse problems for dynamical systems, network inference, and unsupervised type (or role) discovery, with applications in biology, ecology, physics, and social sciences.

\paragraph{Interaction kernel learning.}
A substantial body of literature addresses the recovery of interaction kernels from trajectories in \emph{homogeneous} IPS. Strong form nonparametric regression is applied in \cite{LMT21,LZTM19pnas}. Weak-form methods such as WSINDy are robust to noise and irregular sampling and have been extended to multiforce IPS, e.g., attractive--repulsive, alignment, and drag forces in biological systems \cite{Messenger2022Lair}. Other approaches include density estimation followed by sparse regression to recover the interaction law \cite{messenger2022learning,pereira2023robust}.  Mean-field IPS learning has also been addressed in \cite{LangLu21id, LangLu22}.

\paragraph{Graph inference and latent type discovery.}
When interaction kernels depend on specific particle pairs, the resulting ``disparate species'' systems have been studied in \cite{DFJ2022, JLL2021}, with motivation from Stein variational gradient descent models \cite{LLLLL2020, LWNIPS2016}. Methods for recovering the underlying network include adjacency-matrix estimation from time series \cite{gaskin2024inferring}. 
Latent type or role identification has been pursued through model-agnostic, dimension-reduced clustering of trajectories \cite{TanMiles24}, nonparametric tests for heterogeneity in animal groups \cite{SchaerfTM2021Asmf}, and eco-evolutionary game dynamics on community-structured networks \cite{Masuda24Evolutionary}. 
However, these approaches typically address graph recovery, kernel estimation, and type classification \emph{individually}, and rarely provide a unified framework that integrates all three tasks.

\paragraph{Kernel Heterogeneity in IPS}
When heterogeneity in IPS is assumed, different structures may govern how interaction kernels are assigned. 
One common case is \emph{species-based} heterogeneity, where the kernel depends only on the species of the interacting agents; a more general variant is \emph{ordered-pair} heterogeneity, where the kernel depends on the ordered type pair $(i,j)$, leading to $Q = P^2$ possible kernels. In \cite[Section 4]{LZTM19pnas} and \cite{LMT21_JMLR}, the authors consider the inverse problem of learning interaction laws in systems with multiple types of agents \emph{with the ordered-pair structure known in advance}.

Alternatively, the kernel may depend solely on the \emph{receiving} particle (\emph{row-wise} dependence) or solely on the \emph{forcing} particle (\emph{column-wise} dependence), both yielding $Q = P$ kernels; these two cases are discussed in \cite[Sec.~4.3]{LWLM2024}. 
Our work targets the \emph{general} setting with arbitrary, unknown type matrix $\kappa$ and unknown $Q$, allowing fully heterogeneous interaction patterns without prior structural assumptions.

\paragraph{Complete networks and general networks}
As a special case, when the interaction graph is complete (see Remark~\ref{rmk:complete_network}) or otherwise known, substantial prior work has focused on learning the interaction kernel \cite{LMT21, messenger2022learning, Messenger2022Lair}. 
In such simplified settings, it is unsurprising that estimating both the type matrix and kernel coefficients is typically more straightforward. 
We also provide a dedicated algorithm for the complete-graph case (see Appendix~\ref{apdx:complete_networks}), where clustering can be performed directly on the unnormalized rows without first distinguishing between zero and nonzero interactions. In contrast, for general networks, zero entries must be identified prior to clustering. Our framework is designed to accommodate both settings.

\paragraph{Low-rank structure and mixed-integer challenges.}
In our setting, the factorization \eqref{e:bZi} for $\bZ_i$ induces a rank-$Q$ structure in the stacked $\bZ$, enabling a reduction to low-rank \emph{matrix sensing} under RIP conditions \cite{LWLM2024}. 
This connects to low-rank optimization methods \cite{BM2003,BM2005} and theory for matrix sensing with RIP \cite{BahmaniRomberg17a,Chandrasekher2022alternating, GJZ2017, LeeStoger2022, RFP2010}.

\medskip
\noindent\textbf{Summary.}
Prior work has advanced kernel estimation, graph recovery, and heterogeneity detection separately, often under simplifying assumptions (known types, dense graphs). 
We contribute a unified, scalable framework for \emph{joint} recovery of network topology, multi-type interaction kernels, and latent type assignments from noisy trajectory data, with performance guarantees based on RIP and cluster separability. 

\subsection{Notation}
We summarize the notation used throughout the paper here. 

\paragraph{Scalars, vectors, matrices, and norms.}
Scalars are denoted by plain letters (e.g., $a$), vectors and matrices by bold letters (e.g., $\bx$, $\bA$). 
For a vector $\bx \in \R^d$, we use $\|\bx\|$ to denote its Euclidean norm, and occasionally write $\|\bx\|_{\R^d}$ to emphasize the dimension. 
For a scalar $a \in \R$, we write $|a|$ for its absolute value. 
For a matrix $\bA$, we use $\|\bA\|_2$ for the operator $2$-norm, $\|\bA\|_{2,\infty}$ for the $(2,\infty)$-norm, and $\|\bA\|_F$ for the Frobenius norm.

\paragraph{Block structure.}
For matrices with $N(N-1)$ rows, e.g., $\bZ \in \R^{N(N-1)\times K}$, we denote by $\bZ_i \in \R^{(N-1)\times K}$ the $i$-th block of rows, and by $\bZ_{i,j} \in \R^{1\times K}$ the $j$-th row of $\bZ_i$ (bold since it is a vector).
For a square matrix $\ba \in \R^{N\times N}$, we write $\ba_i$ for its $i$-th row, and $a_{ij}$ for its $(i,j)$-th entry (plain since it is a scalar).
Thus, scalars (e.g., $a_{ij}$) are distinguished from block-rows (e.g., $\bZ_{i,j}$) by font and indexing.

\paragraph{Estimates and ground truth.}
True quantities are denoted with a superscript star (e.g., $\bZ^*$, $\ba^*$), while their estimators are denoted with a hat (e.g., $\widehat{\bZ}$, $\widehat{\ba}$).

\begin{table}[]
\begin{center}
\caption{Notations for the indices and matrices}
\begin{tabular}{l|l}
\toprule
\hline
$[N]$, etc.: index set $\{1,\dots,N\}$, etc & $\ba\in \R^{N\times N}$:graph matrix                        \\ 
$i,j\in[N]$: index for agents               & $\bkappa\in [Q]^{N\times N}$: type matrix                    \\ 
$k\in[K]$: index for basis of kernel        & $\bc\in\R^{K\times Q}$: coefficient matrix                  \\ 
$m\in[M]$: index for samples                & $\bKa\in\{0, 1\}^{N(N-1)\times Q}$: assignment mattrix        \\ 
$l\in[L]$: index of time instants           & $\bu \in \R^{N(N-1)\times Q}$:graph-type feature matrix     \\ 
$q\in[Q]$: index for kernel types           & $\bZ\in \R^{N(N-1)\times K}$: kernel-graph embedding matrix \\ 
$\mathbbm{1}_{\{\cdots\}}$: indicate function           & $\bf1_Q\in \R^{Q}$: vector of all ones with length $Q$ \\
\hline 
\bottomrule
\end{tabular}
\end{center}
\end{table}

The paper is organized as follows. We first introduce the learning framework in Section \ref{Sec:Framework}. Then we provide the detailed algorithms in Section \ref{Sec:Alg}, with theoretical guarantee proved in Section \ref{Sec:Theory}. We provide numerical examples in Section \ref{Sec:Numerics} and conclude in Section \ref{Sec:Conclusion}.

\section{Learning Framework}\label{Sec:Framework}
The joint estimation of the parameter triplet $(\ba^*, \bc^*, \bkappa^*)$ yields a non-convex and mixed-integer optimization problem that is NP-hard in the worst case. To address this, we split the problem defined by \eqref{Eq:Estimator_Loss} into a matrix-sensing subproblem followed by a clustering subproblem. This reformulation reveals the structure of the problem and enables tractable relaxations, forming the basis for the algorithms in the next section.

\subsection{Constraints on network, interaction kernels and type assignments}\label{sec:assignmentMat}

We first assume that the row-normalized \emph{graph matrix} $\ba$ is in the admissible set $\mathcal{M}$: 
\begin{equation}\label{Def:aMat_set}
	\mathcal{M} := \bigg\{\ba=[a_{ij}]\in  [0,1]^{N\times N}\!:\,\,\, a_{ii}=0\,\,,\,\, |\ba_{i}|^2=\sum_{j \neq i} a_{ij}^2 = 1 \,,   \,\forall i\in[N]\bigg\}\,.
\end{equation}
This normalization removes the trivial non-uniqueness of $\ba$ caused by rescaling; other row-normalization schemes are also possible (see the discussion in \cite[Section 1.1]{LWLM2024}).

We now present a pivotal reformulation that decouples the kernel coefficients $\bc$ from the type matrix $\bkappa$. Observe that
$c_{k}^{(\kappa_{ij})}
= \sum_{q=1}^Q c_{k}^{(q)}\,\mathbbm{1}_{\{\kappa_{ij}=q\}}$, which suggests the introduction of $N$ binary \emph{assignment matrices} $\{\bKa_i\}_{i=1}^N$:
 \begin{align}\label{Def:Assign_bKa}
	\bKa= [\bKa_1, \cdots, \bKa_N]^\top\,, 
	\quad \text{where }
    \bKa_i  =\bigl(  \bKa_i \bigr)_{jq} = \bigl( \mathbbm{1}_{\{\kappa_{ij} = q\}} \bigr) \in \{0,1\}^{(N-1) \times Q}, 
\end{align}
with the rows ordered to omit the diagonal (i.e., when $i=j$) entries.  
This allows us to express the kernel-type coefficients compactly as $\bigl(c_{k}^{(\kappa_{ij})}\bigr)_{1 \leq k \leq K} = \bKa_i \bc^\top$. This reformulation naturally leads to a matrix sensing problem, which we address in detail in the following section.

Each of the assignment matrices $\{\bKa_i\}$ has several appealing properties: it is sparse, with each row containing a single one and all other entries zero. 
We demonstrate the construction of the assignment matrix in Example \ref{Example:Type_Mat}. 
By introducing the set
\begin{align}\label{Def:Bi_state}
	\calB&=\Big\{B \in   \{0,1\}^{(N-1)\times Q} 
	\text{ and } B \bf1_Q = \bf1_{N-1}  \Big\},
\end{align}
where $\bf1_Q$ and $\bf1_{N-1}$ denote vectors of ones of length $Q$ and $N{-}1$, respectively.  These constraints above mean that $\bKa$ belongs to the admissible set $\calB^N$. 
\definecolor{myRed}{RGB}{217,83,25}
\definecolor{myYellow}{RGB}{237,177,32}

\tikzset{
style green/.style={
set fill color=green!50!lime!60,
set border color=white,
},
style cyan/.style={
set fill color=cyan!90!blue!60,
set border color=black!80!white,
},
style orange/.style={
set fill color=myRed!80!red!60,
set border color=white,
},
style red/.style={
set fill color=myRed!30!red!30!,
set border color=white,
},
style yellow/.style={
set fill color=myYellow!30!,
set border color=white,
},
}

\tikzset{mycalloutstyle/.style={
       rectangle callout, rounded corners, align=center, text width=3.5cm, fill=myRed!30!red!30!,  callout absolute pointer = (18:1.6)
    }
}
\tikzset{mycalloutstyle2/.style={
       rectangle callout, rounded corners, align=center, text width=3.5cm, fill=myYellow!30!,  callout absolute pointer = (-50:-1.2)
    }
}

\begin{example}\label{Example:Type_Mat}
	Consider the same setting in Example \ref{Exa:Heter_interaction} with $N=5$ agents and $Q=3$ types of interactions. Recall the type matrix $\bkappa$ in \eqref{Eg:agentAssign}. Therefore, the compensating assignment matrices $\bKa_1$ and $\bKa_2$ are given as the Figure \ref{Fig:Code-Assign}.
\begin{figure}[htb]
\centering
\begin{tikzpicture}
\draw (0,0)  node (myequation) {$\bkappa = \begin{bmatrix}
    \tikzmarkin[style yellow]{row1} *  & 2  & 3 & 1 & 1 \tikzmarkend{row1} \\
    \tikzmarkin[style red]{row2}1  & *  & 1 & 3 & 2 \tikzmarkend{row2}\\
    2  & 1  & * & 2 & 3\\
    1  & 3  & 3 & * & 3\\
    2  & 1  & 1 & 1 & *
  \end{bmatrix}$};
\draw (myequation.east) ++ (2.5,0) node[mycalloutstyle]{$\bKa_2=\left(\begin{array}{ccc}
    1  & 0  & 0  \\
    1  & 0  & 0  \\
    0  & 0  & 1  \\
    0  & 1  & 0 
  \end{array}\right)$};
\draw (myequation.west) ++ (-2.5,0) node[mycalloutstyle2]{$\bKa_1=\left(\begin{array}{ccc}
    0  & 1  & 0  \\
    0  & 0  & 1  \\
    1  & 0  & 0  \\
    1  & 0  & 0 
  \end{array}\right)$};
\end{tikzpicture}
\caption{From the type matrix $\bkappa$ to assignment matrices $\bKa_1$ and $\bKa_2$. The diagonal element of $\bkappa$ is not important since self-interactions are not allowed.}
\label{Fig:Code-Assign}
\end{figure}
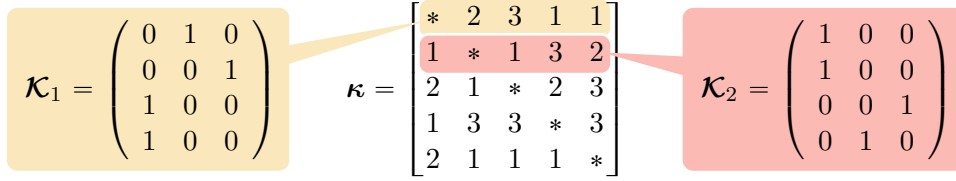
\end{example}

\subsection{Matrix sensing stage}\label{sec:decomposition}
To reformulate the learning problem \eqref{Eq:Estimator_Loss} into a matrix sensing problem, we first transform the drift term into a matrix sensing representation. For the stacked state vector $\bX_t \in \R^{dN}$ in \eqref{eq:stack_X}, we introduce the three dimensional \emph{sensing tensors} $\bB_i(\bX_t)$ built from the basis functions $\{\psi_k\}_{k = 1}^K$: 
\begin{equation}\label{Def:basis_array}
	\begin{aligned}
	\bB_i(\bX_t):= (\bB_i^{jk}(\bX_t) )_{jk} \in \R^{(N-1)\times K\times d}, \quad \bB_i^{jk}(\bX_t)= \psi_k(X_t^j-X_t^i) \in \R^d, 
	\end{aligned}
\end{equation}
where, following the convention, we omit the self-interaction term $j = i$.

We introduce the \emph{graph-type feature matrices} 
\begin{equation}\label{eq:bu_i}
	\bu_i=\Diag(\ba_{i}) \bKa_i\in \R^{(N-1)\times Q}\,, \quad \forall~ i\in[N]
\end{equation}
and the \emph{kernel-graph embedding matrices}
\begin{align}\label{eq:ZUc}
	\bZ_i := \Diag(\ba_{i}) \bKa_i \bc^\top=  \bu_i\bc^\top \in \R^{(N-1)\times K}\,, \quad \forall~ i\in[N]\,,
\end{align}
where $\Diag(\ba_{i}) \in \R^{(N-1)\times (N-1)}$ is the diagonal matrix of the adjacent vector $\ba_{i}$ with the convention removing $j=i$. Using the $i$-th assignment matrix $\bKa_i^{jq} = \mathbbm{1}_{\{\bkappa_{ij}=q\}}$, we can write
\begin{equation*}
	a_{ij} \sum_{q=1}^{Q}  c_{k}^{(q)} \mathbbm{1}_{\{\bkappa_{ij}=q\}}  = a_{ij}\sum_{q = 1}^Q \bKa_i^{jq} \bc^{(q)}_k = \Diag(\ba_{i}) \bKa_i \bc^\top = \bZ_i\,.
\end{equation*}
Together with the above notations, we transform the drift term in the model \eqref{Eq:ips_K_type} as follows, 
\begin{align*}
	 \sum_{j\neq i}  a_{ij}\sum_{k=1}^{K} c_{k}^{(\bkappa_{ij})}\psi_k(X^i_t-X^j_t)  
	&=\sum_{j\neq i} \sum_{k=1}^{K} \left(\sum_{q=1}^{Q} a_{ij}  c_{k}^{(q)} \mathbbm{1}_{\{\bkappa_{ij}=q\}} \right)\psi_k(X^i_t-X^j_t) \\
	&  =\innerp{\bB_i(\bX_t), \bZ_i}_F, \quad \forall i\in [N],
\end{align*} 
where $\innerp{\cdot, \cdot}_F$ represents the Frobenius inner product in $\R^{(N-1)\times K}$ that operates on the first two dimensions of the three-dimensional tensor $\bB_i(\bX_t) \in \R^{(N-1)\times K \times d}$ with the matrix $\bZ_i \in \R^{(N-1)\times K}$, yielding a vector in $\R^d$.

Note that each matrix $\bZ_i$ has low rank: $\mathrm{rank}(\bZ_i)\leq Q(<N)$ since $\bZ_i= \bu_i \bc^\top$ and $\bu_i \in \R^{(N-1)\times Q}$. 
Stacking the matrices in
$	\bZ := [\bZ_1, \cdots, \bZ_N]^\top \in \R^{N(N-1)\times K}
$	and 
$	\bu := [ \bu_1,\cdots, \bu_N	]^\top \in \R^{N(N-1)\times Q} 
$, we notice that 
$$
\bZ = \bu \bc^\top, 
$$
which implies that the matrix $\bZ$ has low rank: $rank(\bZ) \leq Q$.

Therefore, we can split the learning problem in \eqref{Eq:Estimator_Loss} into the matrix sensing subproblem of estimating $\bZ= [\bZ_1, \cdots, \bZ_N]^\top$, followed by the subproblem of factorizing $\bZ_i$ into $(\ba_i, \bKa_i, \bc)$ in \eqref{eq:ZUc}. The matrix sensing subproblem is 
\begin{align}
	&\widehat \bZ = \argmin_{rank(\bZ) \leq  Q\, 
	} \calL(\bZ),\quad \calL(\bZ) = \frac 1 N \sum_{i=1}^N \calL_i(\bZ_i),  \label{eq:matrix_sensing_rank_Q}\\
	&\calL_i(\bZ_i) 
	= \frac{1}{LM}\sum_{l, m = 1}^{L, M}\norm{\dot X_{t_l}^{i, m} - \innerp{\bB_i(\bX_{t_l}^m), \bZ_i}_F}_{\R^d}^2, \label{Def:L_i}
\end{align}
where the \emph{sensing tensors} $\bB_i(\bX_{t_l}^m)$ are defined in \eqref{Def:basis_array}.

We can draw on a rich body of theory and algorithms from the matrix-sensing literature.  On the theoretical side, the Restricted Isometry Property (RIP) conditions ensure that one can recover low-rank matrices via provably convergent optimization routines, even from limited data \cite{BahmaniRomberg17a, Chandrasekher2022alternating, GJZ2017, LeeStoger2022, RFP2010}.  

A variety of solvers may be applied to the matrix-sensing problem \eqref{eq:matrix_sensing_rank_Q}, including nuclear norm minimization, alternating least squares, and singular value thresholding. In particular, knowing the target rank makes the alternating least squares algorithm especially effective in practice.  Alternatively, one can forgo the explicit low-rank constraint and simply apply ordinary least squares, an approach explored in depth in \cite{LWLM2024}.  We defer the full description of these algorithms to Section \ref{Subsec:alg_matrix_sensing}.

\subsection{Clustering stage}\label{sec:clustering}
In the second stage, we recover the kernel types $\bKa$ from the estimated tensor $\widehat{\bZ}$ in \eqref{eq:matrix_sensing_rank_Q} by casting it as a clustering problem.  Recall from \eqref{eq:bu_i} that each $\bu_i$ is a sparse matrix with at most one nonzero entry per row.  Consequently, the rows of $\widehat{\bZ}_i$ can point in at most $Q$ distinct directions (see Example \ref{example-Zi} below). This observation motivates clustering the rows of $\widehat{\bZ}_i$ by their normalized direction vectors.

\begin{example}\label{example-Zi}
	Consider the same setting in Example \ref{Exa:Heter_interaction} with $N=5$ agents and $Q=3$ types of interactions. Recall the type matrix of the first particle, as in Figure \ref{Fig:Code-Assign}; we can multiply it by the adjacency row vector of the first particle $\Diag(\ba_{1})$, so that 
	\begin{equation}\label{Exa2:calK_1}
		\bKa_1=
		\left[
		\begin{array}{ccc}
		    0  & 1  & 0  \\
		    0  & 0  & 1  \\
		    1  & 0  & 0  \\
		    1  & 0  & 0 
		\end{array}
		\right], \quad 
	\bu_1=\Diag(\ba_{1}) \bKa_1 
	= \left[\begin{array}{ccc}
    0  & a_{12}  & 0  \\
    0  & 0  & a_{13}  \\
    a_{14}  & 0  & 0  \\
    a_{15}  & 0  & 0 
 	\end{array}\right]\,,
	\end{equation}
	which is a sparse matrix. 
	Therefore, the true $\bZ_1$ matrix is
	\begin{equation}\label{Exa2:Z_1}
		\bZ_1 =  \bu_1\bc^\top 
		= \begin{bmatrix}
			a_{12} (\bc^{(2)})^\top \\ a_{13} (\bc^{(3)})^\top \\ a_{14} (\bc^{(1)})^\top \\ a_{15} (\bc^{(1)})^\top
		\end{bmatrix}
		=
		\left[
		\begin{array}{cccc}
		    a_{12} c^{(2)}_1 & a_{12} c^{(2)}_2 & \cdots & a_{12} c^{(2)}_K  \\
		    a_{13} c^{(3)}_1 & a_{13} c^{(3)}_2 & \cdots & a_{13} c^{(3)}_K  \\
		    a_{14} c^{(1)}_1 & a_{14} c^{(1)}_2 & \cdots & a_{14} c^{(1)}_K  \\
		    a_{15} c^{(1)}_1 & a_{15} c^{(1)}_2 & \cdots & a_{15} c^{(1)}_K 
		\end{array}
		\right]\,. 
	\end{equation}
	Ideally, the estimate $\widehat{\bZ}_1$ will be close to the true matrix ${\bZ}_1$, which consists of four rows in the directions of $\{\bc^{(i)}\}_{i=1}^3$. Consequently, we expect that the last two rows of $\widehat{\bZ}_1$ can be grouped into a single cluster.
\end{example}

Nevertheless, when two interaction kernels are too similar or the associated network weights are small, they become difficult to distinguish, particularly in the presence of noise. To enable reliable clustering, certain parameters governing the separation, arising from both the kernels and the weight matrix, must be sufficiently large. We first define the \emph{oracle minimal inter-class angle} of the true kernel coefficients $\bc^* = [\bc^{*(1)}, \cdots, \bc^{*(Q)}]$ as
    \begin{equation}\label{eq:theta0}
    	\theta_{min} := \min_{1 \leq p < q \leq Q} \arccos \frac{\innerp{\bc^{*(p)}, \bc^{*(q)}}}{\|\bc^{*(p)}\| \|\bc^{*(q)}\|}.
    \end{equation}
We then define the \emph{oracle minimal row $2$-norm} of the ground truth matrix $\bZ^*$,
	\begin{equation}\label{eq:z0}
		z_{min}  = \min_{1 \leq i, j \leq N, a_{ij} \neq 0}\norm{\bZ^*_{i,j}}\,,
	\end{equation}
	where $\bZ^*_{i,j}$ represents the $j$-th row of the $i$-th block $\bZ^* _i$. Note that $z_{min}$ relates the minimal nonzero entry of the graph matrix and the minimal kernel coefficient norms through
	\begin{equation}\label{eq:z0_a0_c0}
		z_{min} \geq \left(\min_{1 \leq i, j \leq N, a^*_{ij} \neq 0}a^*_{ij} \right)  \cdot \left( \min_{1 \leq q \leq Q} \|\bc^{*(q)}\| \right).
	\end{equation}
See Figure~\ref{fig:Exm_threshold} for the effect of these parameters. The conditions on $\theta_{min}$ and $z_{min}$ to guarantee separability of different kernel types are presented in Section \ref{sec:separation}.

\begin{figure}[htbp]
\center
\includegraphics[width=0.3\textwidth]{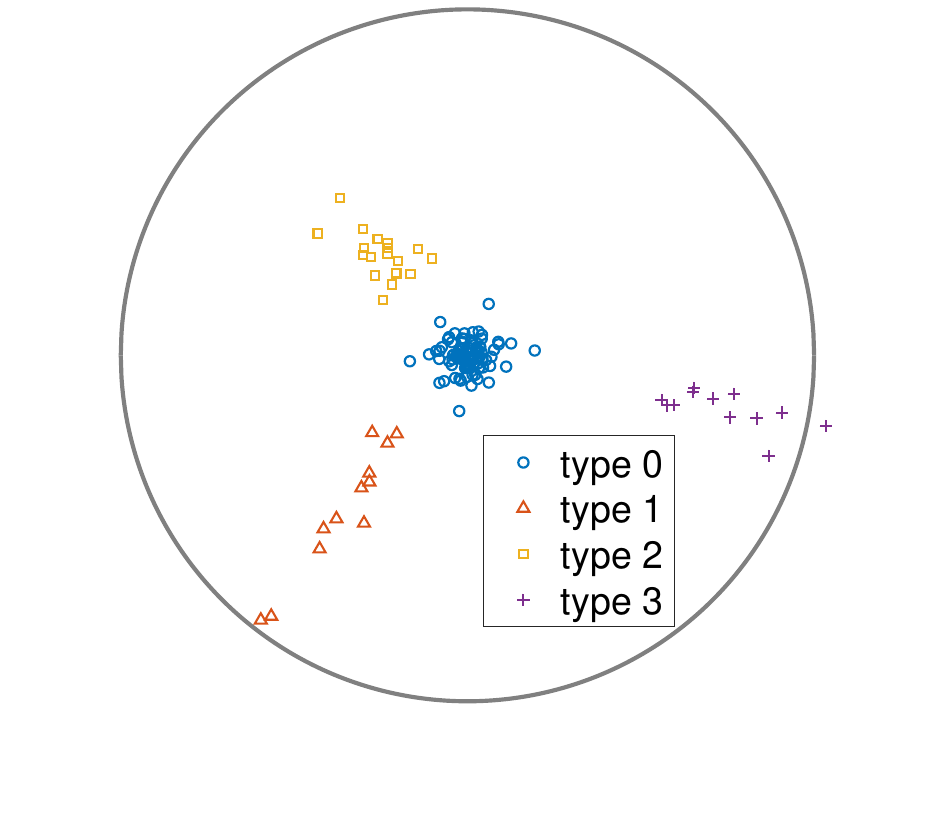}
\includegraphics[width=0.3\textwidth]{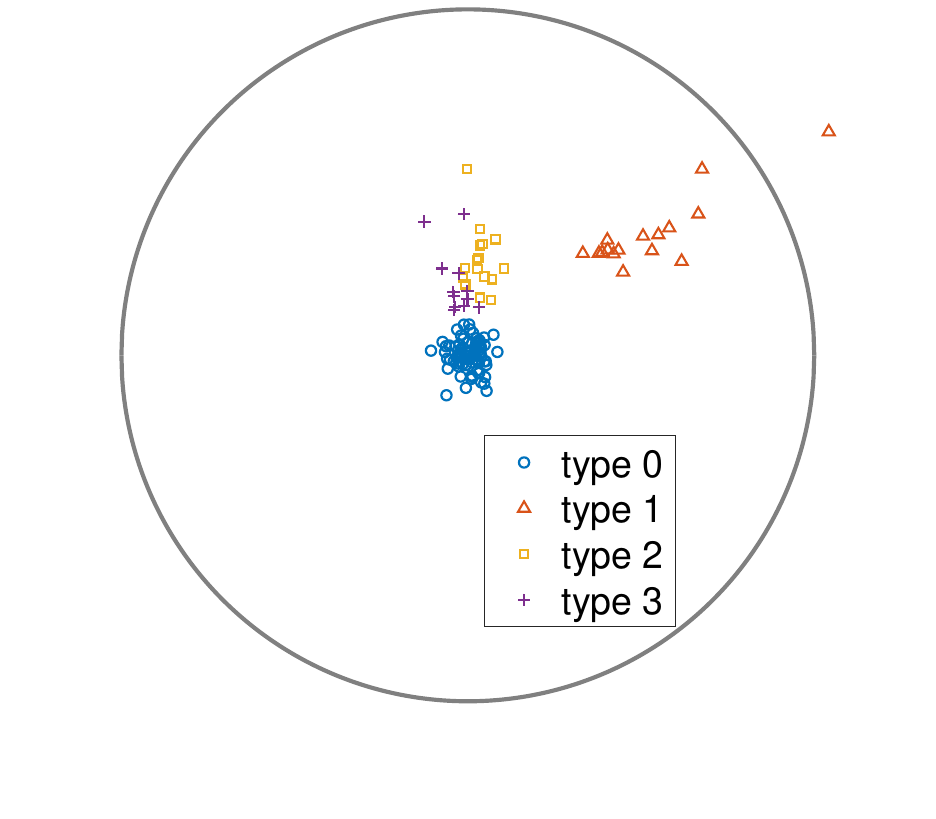}
\includegraphics[width=0.3\textwidth]{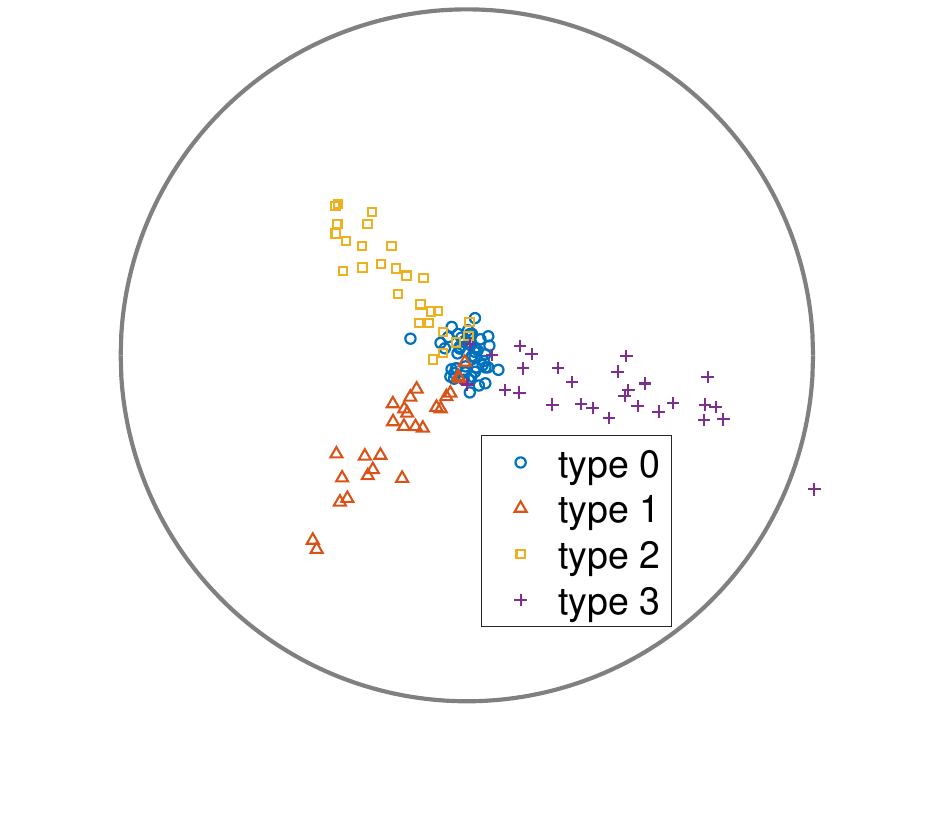}
\caption{Examples of row vectors from the matrix $\widehat{\bZ}$ estimated from noisy data over 2D hypothesis space ($K=2$), with $Q=3$ types of nonzero rows. Type 0 corresponds to a $0$ entry on the graph. \textbf{Left:} All types are clearly distinguishable due to sufficiently large angle $\theta_{min}$ and row norm $z_{min}$. \textbf{Middle:} Nonzero rows of different types become indistinguishable due to $\theta_{\min}$ being small. \textbf{Right:} Zero and nonzero rows become indistinguishable due to $z_{\min}$ being small.}
\label{fig:Exm_threshold}
\end{figure}

We will apply $K$-means clustering for the normalized rows of $\widehat\bZ$ to estimate the assignment matrix $\{\bKa_i\}$, and therefore, equivalently, the type matrix $\bkappa$. The full clustering algorithm is detailed in Section \ref{Subsec:K-mean}. Given prior information on the parameters $\theta_{min}$ and $z_{min}$, the algorithm's performance depends quantitatively on the noise level and the amount of available data, as analyzed in detail in Section~\ref{sec:separation}.

\subsection{Matrix factorization stage}
In the third stage, we recover the weight matrix $\ba$ and the coefficients matrix $\bc$ by factorizing $\widehat{\bZ}$ with the kernel type matrix $\widehat{\bkappa}$ by a linear regression including the constraint \eqref{Def:aMat_set}. 

First, with $\bZ_i$ in \eqref{eq:ZUc}, recall that each row of $\bZ_i = \Diag(\ba_{i}) \bKa_i \bc^\top=  \bu_i\bc^\top$ is a scaled version of a corresponding column of $\bc$. Denoting the norm of the kernel coefficient vector by $v_q = \|\bc^{(q)}\|$, $\bc$ can be reconstructed from the normalized kernel coefficient vectors $\bv_q$ by
\begin{equation}\label{eq:cvv}
	\bc = [v_1 \bv_1^\top, \dots, v_Q\bv_Q^\top] \in \R^{Q \times K}.
\end{equation}
Moreover, the norms $\{v_{q}\}_{q=1}^Q$ are determined from the linear system 
\begin{equation}\label{eq:LS_eta}
	\bA \bEta = \bf1_{N}, 
\end{equation}
where $\bf1_N$ is the $N$-dimensional vectors of ones and 
\begin{equation}\label{eq:A_eta}
\bEta = 
\begin{bmatrix}
	1/v_1^2 \\ 
	\vdots \\
	1/v_Q^2
\end{bmatrix} \in \R^Q, \quad
	\bA = 
	\begin{bmatrix}
		\bA_{1}\\ \vdots \\ \bA_{N}	
	\end{bmatrix} \in \R^{(N-1) \times Q},\quad
	\bA_{i} = [ z_{i, 1}^2, \dots,  z_{i, N}^2] \bKa_i \in \R^{1\times Q}.
\end{equation}
Eq.\eqref{eq:LS_eta} is derived as follows. Note that 
\begin{equation}\label{eq:z-a-v}
z_{i, j} = a_{ij}{v}_{\kappa_{ij}}
\end{equation}
 is the row norm of $\bZ$ when $i\neq j$, and the indices with $i = j$ can be omitted since $a_{ii}=0$. The row-normalization constraint on $\ba$ implies
\begin{equation*}
	1 = \sum_{j = 1, j \neq i}^{N}a_{ij}^2 = \sum_{j = 1, j \neq i}^{N} \frac{z_{i, j}^2}{v_{\kappa_{ij}}^2}, \quad i=  1, \dots, N, 
\end{equation*}
which yields \eqref{eq:LS_eta}.

Therefore, given the estimated assignment matrix $\widehat {\bKa}$ (hence type matrix $\widehat\bkappa$) and cluster centers $\{\widehat\bv_q\}$, we can  first solve $\{v_{q}\}$ from \eqref{eq:LS_eta} via linear regression, and then obtain $\bc$ from \eqref{eq:cvv} and $\ba$ from \eqref{eq:z-a-v}. 
Note that such a factorization of the estimated matrix $\widehat\bZ$ into the graph matrix $\ba$ and the kernel coefficients $\bc$ has negligible computational cost. Finally, a post-processing ALS refinement is performed to enhance robustness to noise. The details of the factorization and the post-processing are presented in Section \ref{sec:post_processing}.

\begin{remark}[Complete network]\label{rmk:complete_network_2}
Much effort has been devoted to studying heterogeneous interacting particle systems (IPS) without considering the underlying interaction network (cf. \cite{LMT21_JMLR, messenger2022learning} and references therein). These works typically assume a complete network, i.e., $a_{ij} \equiv 1/N$. In this setting, multi-type interaction kernels can be estimated without prior knowledge of particle types, using the same framework developed in this manuscript. 
To see this, consider the same type of matrix of $1$-st particle $\bKa_1$ as given in  \eqref{Exa2:calK_1}. Then, the corresponding matrices $\bu_1$ and $\bZ_1$ in \eqref{Exa2:calK_1} and \eqref{Exa2:Z_1} become
\begin{equation*}
	\bu_1 = \frac{1}{N}\left[\begin{array}{ccc}
    0  & 1  & 0  \\
    0  & 0  & 1  \\
    1  & 0  & 0  \\
    1  & 0  & 0 
  \end{array}\right]
  \quad \text{and} \quad 
  \bZ_1=
  \begin{bmatrix}
	(\bc^{(2)})^\top \\ (\bc^{(3)})^\top \\ (\bc^{(1)})^\top \\ (\bc^{(1)})^\top
  \end{bmatrix}
  =
  \left[
		\begin{array}{cccc}
		    c^{(2)}_1 & c^{(2)}_2 & \cdots & c^{(2)}_K  \\
		    c^{(3)}_1 & c^{(3)}_2 & \cdots & c^{(3)}_K  \\
		    c^{(1)}_1 & c^{(1)}_2 & \cdots & c^{(1)}_K  \\
		    c^{(1)}_1 & c^{(1)}_2 & \cdots & c^{(1)}_K 
		\end{array}
 \right]\,.
\end{equation*}
Our three-stage optimization framework yields improved and more robust clustering performance for this class of problems. 
We can apply \emph{positional clustering} to the rows of $\bZ_1$ in the second stage of our framework. 
The key advantage is its flexibility: our framework relaxes the constraint that $\bu_1$ must conform to a fixed matrix structure in the complete network setting. This added flexibility can improve robustness.
\end{remark}


\section{Algorithms}\label{Sec:Alg}
Our main algorithm mirrors the three stages described above: the first two stages relax the mixed-integer optimization problem into a matrix sensing problem and a clustering problem, and the third stage of factorization and post-processing yields the final estimators in a robust fashion.

\begin{enumerate}
\item \textbf{Matrix sensing:} Estimate the kernel-graph embedding matrices matrices $\{\bZ_i\}$ defined in \eqref{eq:ZUc}. We employ the Alternating Least Squares (ALS) algorithm, an approach similar to factorization techniques commonly used in low-rank matrix sensing (e.g., \cite{BM2003, BM2005}). 

\item  \textbf{Clustering:} Estimate the type matrix $\bkappa$ by clustering of the normalized rows of $\{\bZ_i\}$. The number of clusters can either be given or estimated during clustering. 

\item  \textbf{Factorization and postprocessing:} Estimate $\{\ba_{i}\}$ and $\bc$ using the learned type matrix $\widehat\bkappa$. One can directly solve them by the row-normalization constraints on $\{\ba_i\}$. To further improve performance, we take another step by using ALS to estimate $\{\ba_{i}\}$ and $\bc$ based on the $\widehat\bkappa$. 
\end{enumerate}

Moreover, Section \ref{sec:alg_extensions} discusses variant implementations and extensions of our main algorithm, including the case of a complete network and a 3-way ALS algorithm for jointly learning $(\ba,\bkappa,\bc)$.

\subsection{Matrix sensing}\label{Subsec:alg_matrix_sensing}
In the first stage, we treat the rank-$Q$ matrix sensing problem \eqref{eq:matrix_sensing_rank_Q}. If the true value of $Q$ is known, we use it directly; otherwise, we initialize with an upper bound and we will then refine it iteratively.

Recall that for each $i\in[N]$, we have  $\bZ_i = \bu_i \bc^\top$ in \eqref{eq:ZUc}. In this representation, the \emph{graph-type feature} matrix $\bu_i$ encodes the adjacency and type information for the $i$-th particle, while the \emph{coefficient matrix} $\bc$ represents the $Q$ distinct interaction kernels. Also, note that the loss function $\calL_i(\bZ_i)=\calL_i(\bu_i,\bc)$ in \eqref{Def:L_i} is quadratic in $\bu_i$ when $\bc$ is fixed and vice versa. Thus, we can use an Alternating Least Squares (ALS) scheme to estimate $\bZ_i$: starting from an initial estimate $\widehat{\bu}$, we alternatively update $\widehat{\bc}$ and $\widehat{\bu}$ using least squares until convergence.
\begin{enumerate}[leftmargin=5mm]
	\item[$\bullet$] \textbf{Update $\bc$:} With an estimate $\widehat{\bu}$ fixed, we update $\bc$ by solving the following least-squares problem: 
	\begin{equation}\label{ALS:U_to_V}
	\min_{\bc\in \R^{K\times Q}} \calL(\widehat{\bu}, \bc) = \frac{1}{LMN}\sum_{l,m,i=1}^{L,M,N} \norm{\dot{X}_{t_l,m}^{i} -  \innerp{\bB(\bX_{t_l}^m)_i, \widehat{\bu}_i \bc^\top}_F}^2_{\R^d}\,.
	\end{equation}
	\item[$\bullet$] \textbf{Update $\bu$:} With the new estimate $\widehat{\bc}$ fixed, we update each $\bu_i$ for $i\in[N]$ by solving its corresponding least-squares problem:
	\begin{equation}\label{ALS:V_to_U}
	\min_{\bu_i\in \R^{(N-1)\times Q}} \calL_i(\bu_i, \widehat{\bc}) = \frac{1}{LM}\sum_{l,m}^{L,M} \norm{\dot{X}_{t_l,m}^{i} -  \innerp{\bB(\bX_{t_l}^m)_i, \bu_i \widehat{\bc}^\top}_F}^2_{\R^d}. 
	\end{equation}
\end{enumerate}

The above steps are repeated until the relative change in the parameters falls below a prescribed threshold $\varepsilon$ or a maximum number of iterations $n_{maxiter}$ is reached. 
Note that the same estimate $\widehat{\bc}$ is used to update all $\{\bu_i\}_{i=1}^N$ in parallel, and the collection $\{\widehat{\bu}_i \}_{i=1}^N$ is then used to update $\bc$. The complete procedure is summarized in Algorithm \ref{alg:ALS}. 
The well-posedness of the matrix sensing problem and guarantees on the effectiveness of related algorithms such as ALS rely critically on the Restricted Isometry Property (RIP), as stated in Assumption~\ref{conj_RIP}. 
We postpone the analysis of the convergence of ALS and of the critical data size for identifiability and robust estimation to Section \ref{sec:RIP}.

\begin{remark}[Transformation invariance] The matrix sensing stage does not, in general, recover the factors $\bu_i$ (and hence $\ba_i$, $\bKa_i$) and $\bc$ in the factorization $\bZ_i = \bu_i \bc^\top$ from \eqref{eq:ZUc}. Indeed, for any invertible matrix $S \in \R^{Q \times Q}$, we have $\widehat{\bZ}_i = \widehat{\bu}_i \widehat{\bc}^\top = (\widehat{\bu}_i S)(S^{-1}\widehat{\bc})^\top$, demonstrating an invariance under such transformations. This ambiguity is resolved by imposing the integer constraint $\bKa_i \in \calB$,  which we address in the next stage via $K$-means clustering.
\end{remark}

\paragraph*{Computational complexity for ALS}
The computational complexity of ALS is determined by the cost of solving least-squares problems per iteration and by the number of iterations to convergence. In each iteration, it requires $\mathcal{O}(LMN \cdot d^2QK)$ operations to update $\bc$ and $\mathcal{O}(LM \cdot d^2QN)$ operations to update each $\bu_i$, therefore the total complexity is $\mathcal{O}(LMN \cdot d^2Q(K + N))$, assuming $Q$ and $d$ are small compared to $N$ and $K$. The number of iterations depend on the RIP condition of the matrix sensing problem. For $Q = 1$, \cite{LeeStoger2022} shows convergence to accuracy $\varepsilon$ in 
$\mathcal{O} \left( \frac{\log N^2 + \log(1/\varepsilon)}{\log \log N^2} \right)$ iterations under suitable RIP conditions. Little seems to be known about the convergence behavior for general $Q > 1$. We summarize the complexity results in Table \ref{tab:ALS_complexity}.

\begin{table}[h]
\centering
\caption{Computational Complexity of ALS}
\begin{tabular}{c|l|p{6.5cm}}
\hline
\hline
\textbf{Component} & \textbf{Complexity} & \textbf{Remarks} \\
\hline
\hline
Update $\bc$ & $\mathcal{O}(LMN \cdot d^2 Q K)$ & Solving least-squares over all samples \\
\hline
Update each $\bu_i$ & $\mathcal{O}(LM \cdot d^2 Q N)$ & Parallel update of all $\bu_i$ \\
\hline
\textbf{Total per iter.} & $\mathcal{O}(LMN \cdot d^2 Q (K + N))$ & Assumes $Q$ and $d$ are small relative to $N$ and $K$ \\
\hline
\textbf{Number of iter. ($Q = 1$)} & $\mathcal{O} \left( \dfrac{\log N^2 + \log(1/\varepsilon)}{\log \log N^2} \right)$ & From \cite{LeeStoger2022}, under RIP condition \\
\hline
\textbf{Number of iter. ($Q>1$)} & Unknown & Convergence remains open when $Q > 1$ \\
\hline
\hline
\end{tabular}
\label{tab:ALS_complexity}
\end{table}

\begin{algorithm}[htb]
\caption{ALS: alternating least squares}
\label{alg:ALS}
\begin{algorithmic}[1] 
\Procedure{Alternating Least Squares (ALS)}{$\{\bX_{t_l}^m\}_{l = 1, m = 1}^{L, M}$, $\{\psi_k\}_{k = 1}^p$, $\varepsilon$, $n_{maxiter}$}
    \State Compute derivative $\{\Delta \bX_{t_l}^m\}$ and sensing tensors $\{\bB(\bX_{t_l}^m)\}_{l, m}$ from basis functions.  
    \If{True $Q$ is not known}
    	\State Choose $Q = \widehat{Q}$ that is relatively large.
    \EndIf
    \State Randomly pick initial value $\widehat{U}_0\in \R^{N(N-1)\times Q}$.
    \For{$\tau = 1, \dots, n_{maxiter}$}
        \State Solve $\widehat{\bc}_\tau=\argmin_{\bc} \calL(\widehat{\bu}, \bc)$ in \eqref{ALS:U_to_V} using least square with $\widehat{\bu} = \widehat{\bu}_{\tau-1}$. 
        \For{$i = 1, \dots , N$}
        	\State Solve $\widehat{\bu}_{\tau, i}=\argmin_{\bu_{i}} \calL(\bu_{i}, \widehat{\bc})$ in \eqref{ALS:V_to_U} using least square with $\widehat{\bc} = \widehat{\bc}_{\tau}$. 
    	\EndFor
        \State Construct the matrix $\widehat{\bu}_\tau = [\widehat{\bu}_{\tau, 1}^\top, \dots, \widehat{\bu}_{\tau, N}^\top]^\top$.
        \State Exit loop if $\|\widehat{\bu}_\tau - \widehat{\bu}_{\tau-1}\|_F \leq \varepsilon \|\widehat{\bu}_{\tau-1}\|_F$ and $\|\widehat{\bc}_\tau - \widehat{\bc}_{\tau-1}\|_F \leq \varepsilon \|\widehat{\bc}_{\tau-1}\|_F$
    \EndFor
    \State \textbf{return} $\widehat{\bZ} = \widehat{\bu}_\tau \widehat{\bc}_\tau$.
\EndProcedure
\end{algorithmic}
\end{algorithm}

\subsection{Clustering}\label{Subsec:K-mean}
In the second stage, we estimate the type matrix $\bkappa$ (equivalently, the assignment matrix $\bKa$) by applying K-means clustering on $\widehat{\bZ}$. The clustering also provides normalized rows of the coefficient matrix $\bc$. 
We begin by considering the case where the true value of $Q$ is known in order to illustrate the core clustering procedure. Subsequently, we introduce an iterative search algorithm designed for the more general case where $Q$ is unknown.

\subsubsection{K-means clustering with a known number of types}\label{subsec:K-means_known}

As discussed in Section~\ref{sec:clustering}, the matrix $\bZ \in \R^{N(N-1)\times K} $ is expected to have $Q$ distinct rows up to scaling. This structure motivates applying clustering to the row-normalized matrix $\widehat{\bZ}$ in order to recover the assignment matrix.

Recall that $\bZ_i \in \R^{(N-1)\times K}$ is the $i$-th block of $\bZ$. Let $\bZ_{i, j}$ denote the $j$-th row of $\bZ_i$ and define the row norm as 
\begin{equation}\label{eq:z_norm}
	z_{i, j} = \|\bZ_{i, j}\|_{\R^{1\times K}}.
\end{equation}
Also, let
\begin{align}\label{eq:Z_normalize_i}
	\overline{\bZ}_{i, j} = \frac{1}{z_{i,j}} \bZ_{i, j} \in \R^{1\times K }, \quad \overline{\bZ}_i = \begin{bmatrix}
 	\overline{\bZ}_{i, 1} \\ \vdots \\ \overline{\bZ}_{i, N}	
 \end{bmatrix} \in \R^{(N-1)\times K}, \quad 
 	\overline{\bZ} = \begin{bmatrix}
 	\overline{\bZ}_{1} \\ \vdots \\ \overline{\bZ}_{N}	
 \end{bmatrix} \in \R^{N(N-1) \times K}, 
\end{align} 
where $(i, j)$ ranges from all possible pairs with the convention to omit $i = j$. We similarly define $\widehat{z}_{i, j}, \overline{\widehat \bZ}_{i, j}$, $\overline{\widehat{\bZ}}_i$ and $\overline{\widehat{\bZ}}$ for the estimated $\widehat{\bZ}$. After normalization, $\overline{\bZ}$ should ideally contain $Q$ distinct row patterns up to scaling. However, noise and estimation errors may distort this structure. In particular, if $a_{ij} = 0$, then the corresponding row of $\bZ$ should be exactly zeros. However, $\widehat{\bZ}$ may contain small nonzero values in these rows due to noise, which can then be amplified by normalization, leading to significant distortion and unreliable clustering. Therefore, it is essential to separate zero and non-zero rows before clustering. 

We first identify index pairs $(i, j)$ corresponding to graph entries that are effectively zero. These are collected in the set
\begin{equation}\label{eq:I_0}
	I_0 = \{(i, j) : \widehat{z}_{i,j} < z_0\},
\end{equation}
where $z_0$ denotes a prescribed threshold. When the prior information $z_{min}$ defined in \eqref{eq:z0} is available, $z_0$ can be selected optimally based on the noise level and the data size; see Section~\ref{sec:separation} for details. In the absence of such prior information, $z_0$ may instead be treated as a tunable parameter and chosen empirically. 

We then apply K-means clustering~\cite{arthur2006k, lloyd1982least} to the remaining rows of $\overline{\widehat\bZ}$ using the cosine distance $d(\bx, \by) := 1-\frac{\innerp{\bx, \by}}{\|\bx\| \|\by\|}$ for any $\bx,\by\in \R^K$. In this angular clustering setting, a cluster center $\bv$ is determined via the normalized average of the points $\{\bx_i\}_{i = 1}^n$ in the cluster:
\begin{equation}\label{cos_dist_center}
\bv = \bu/\|\bu\|, \quad \bu = \frac{1}{n}\sum_{i = 1}^n \bx_i. 
\end{equation}
This procedure yields $Q$ clusters with corresponding index sets $I_1, \dots, I_Q$,
\[
I_q = \big\{(i, j) : \overline{\widehat\bZ}_{i,j} \text{ belongs to the } q\text{-th cluster}\big \},
\]
and cluster centers $\mathbf{v}_1, \dots, \mathbf{v}_Q \in \mathbb{R}^{1\times K}$ that represent estimated normalized kernel coefficients. If $(i,j) \in I_q$, we identify the  interaction between the $i$-th and $j$-th particles to be of the $q$-th kernel type, and we set
\begin{equation}\label{eq:Kmeans_Kappa}
	\widehat{\bKa}_{i,j} = \be_q, \quad \widehat{\kappa}_{ij} = q, 
\end{equation}
for $q = 0,1, \dots, Q$, where $\widehat{\bKa}_{i,j}$ represents the $j$-th row of $\widehat{\bKa}_i$, $\be_q$ is the $q$-th standard basis row vector in $\mathbb{R}^Q$, and $\be_0$ is the zero vector. The inclusion of $\be_0$ in $\widehat\bKa$ reflects unidentifiability of the type $\kappa_{ij}$ when $a_{ij} = 0$. 

\begin{remark}[Cluster ordering]\label{rmk:cluster_order}
	To address label ambiguity in clustering, we sort clusters by their first occurrence in lexicographic order over the index set $(1,2), (1,3), \dots, (N,N-1)$.
\end{remark}

\begin{remark}[Active and inactive kernels]\label{rmk:act_inact_kernel}
Even when the true number of interaction kernels $Q$ is known, not all kernels are necessarily active in the observed dynamics. This may occur if the graph matrix $ \ba $ contains zero entries, i.e., $ a_{ij} = 0 $ for all pairs $ (i, j) $ such that $ \kappa_{ij} = q $ for the $ q $-th kernel, thus making the kernel inactive. Consequently, it is preferable to determine the number of active kernels adaptively from the data. 
In contrast, under the complete network assumption described in Remark~\ref{rmk:complete_network_2}, all interactions are present, and the weight matrix $\ba$ contains no zero entries. In this case, the number of kernels used in the dynamics coincides with the true value of $Q$, and one may directly apply Algorithm~\ref{alg:Kmeans_known_Q}.
\end{remark}

\begin{algorithm}[t]
\caption{Clustering $\widehat{\bZ}$ with known $Q$}
\label{alg:Kmeans_known_Q}
\begin{algorithmic}[1] 
\Procedure{Spherical K-means, $Q$ known}{estimated $\widehat{\bZ}$, $z_{0}$, $Q$}
    \State Row-normalize $\widehat{\bZ}\in \R^{N(N-1)\times K}$ to get $\overline{\widehat\bZ}$ and row norms $\{ \widehat z_{i,j}\}$.
    \State Construct $I_0 = \{(i, j), \widehat z_{i, j} < z_0\}$.
	\State Use K-means to classify the rows of $\overline{\widehat\bZ}$ with $z_{i,j}\geq \delta z_0$ into $Q$ clusters with cosine distance, get index sets $\{I_q\}_{q = 1}^Q$ and cluster centers $\{\mathbf{v}_q\}_{q = 1}^Q$.  
	\State Construct $\widehat{\bKa}\in \{0,1\}^{N(N-1)\times Q}$ and $\widehat{\bkappa} \in [Q]^{N\times N}$ by $\widehat{\bKa}_{i,j} = \be_q$ and $\widehat{\kappa}_{ij} = q$ if $(i, j) \in I_q$ for $q = 0, 1, \dots, \widehat Q$. 
	\State \textbf{return} $\widehat{\bKa},\widehat{\bkappa}, \{I_q\}_{q = 0}^Q, \{\widehat\bv_q\}_{q = 1}^Q$
    \EndProcedure
\end{algorithmic}
\end{algorithm}

\subsubsection{K-means clustering with an unknown number of types}\label{subsec:K-means_unknown}
When the true number of kernel types is unknown, it must be inferred from the data. Let $\widehat{Q}$ denote the current estimate of the number of clusters, with associated clustering results $\{I_q, \widehat\bv_q\}_{q = 1}^{\widehat{Q}}.$ We introduce the \emph{empirical maximal intra-cluster angle} defined by
\begin{equation}\label{eq:D_max}
\Theta = \max_{1 \leq q \leq \widehat{Q}}\,\,\max_{(i, j), (i', j') \in I_q} \,
\arccos\, \innerp{\overline{\widehat\bZ}_{i, j},\overline{\widehat\bZ}_{i', j'}},
\end{equation}
and compare it against a prescribed threshold $\Theta_0$. If $\Theta > \Theta_0$, this indicates that at least one cluster contains elements with significant angular dispersion, suggesting the need to increase the number of clusters.
In addition, we define the \emph{empirical minimal inter-cluster angle} based on the current cluster centers $\{\widehat\bv_q\}_{q = 1}^{\widehat {Q}}$:
\begin{equation}\label{eq:theta}
\theta := \min_{1 \leq p < q \leq \widehat{Q}} \arccos \frac{\innerp{\widehat\bv_p, \widehat\bv_q}}{\|\widehat\bv_p\| \|\widehat\bv_q\|},
\end{equation}
and compare it with a threshold $\theta_0$. If $\theta <\theta_0$, the cluster centers are insufficiently separated, indicating an overestimation of the number of clusters and the need to reduce $\widehat{Q}$.

\begin{algorithm}[t]
\caption{Clustering $\widehat{\bZ}$ with unknown $\widehat{Q}$}
\label{alg:Kmeans_unknown_Q}
\begin{algorithmic}[1]
\Procedure{Spherical K-means, $Q$ Unknown}{$\widehat{\bZ}$, $z_0$, $\theta_0$, $\Theta_{0}$,  ${Q}_0$, $n_{\max}$ }
    \State Row-normalize $\widehat{\bZ}\in \R^{N(N-1)\times K}$ to get $\overline{\widehat\bZ}$ and row norms $\{ \widehat z_{i,j}\}$.
    \State Construct $I_0 = \{(i, j), \widehat z_{i, j} < \delta z_0\}$ 
    \State Initialize $\widehat{Q} = {Q}_0$, \texttt{success} = \texttt{false}
    
    \For{$\tau = 0, \dots, n_{\max}$}
        \State Run K-means with $\widehat{Q}$ clusters on rows of $\overline{\widehat \bZ}$ with $\widehat z_{i,j}\geq z_0$ using cosine distance, get index sets $\{I_q\}_{q = 1}^{\widehat{Q}}$ and cluster centers $\{\widehat\bv_q\}_{q = 1}^{\widehat{Q}}$.  
        \State Compute min inter-cluster angle $\theta$ and max intra-cluster angle $\Theta$
        
        \If{$\theta > \theta_0$ and $\Theta < \Theta_{\max}$ }
            \State \texttt{success} = \texttt{true}, \textbf{break}
        \ElsIf{$\theta < \theta_0$}
            \State $\widehat{Q} \gets \widehat{Q} - 1$
        \ElsIf{$\Theta > \Theta_{0}$}
            \State $\widehat{Q} \gets \widehat{Q} + 1$
        \EndIf
    \EndFor

    \If{\texttt{success} = \texttt{false}}
        \State \textbf{return} Classification fails
    \Else
        \State Construct $\widehat{\bKa}\in \{0,1\}^{N(N-1)\times \widehat{Q}}$ and $\widehat{\bkappa}\in [\widehat{Q}]^{N\times N}$ using the clustering result by setting $\widehat{\bKa}_{i,j} = \be_q$ and $\widehat{\bkappa}_{ij} = q$ if $(i, j) \in I_q$ for $q = 0, 1, \dots, \widehat{Q}$. 
        \State \textbf{return} $\widehat{\bKa},\widehat{\bkappa}, \widehat{Q}, \{I_q\}_{q=1}^{\widehat{Q}}, \{\widehat\bv_q\}_{q=1}^{\widehat{Q}}$
    \EndIf
\EndProcedure
\end{algorithmic}
\end{algorithm}

This procedure enables adaptive adjustment of the number of clusters during the iterative process. The full method is summarized in Algorithm~\ref{alg:Kmeans_unknown_Q}. Compared to Algorithm~\ref{alg:Kmeans_known_Q}, this adaptive approach is generally more robust, particularly when the ideal thresholds $\theta_0$ and $\Theta_{0}$ are available. The prior information $\theta_{min}$, as defined in \eqref{eq:theta0}, can be used to guide the selection of recommended thresholds $\theta_0$ and $\Theta_{0}$. The appropriate values of these thresholds also depend on the noise level and the amount of available data; a detailed analysis is provided in Section~\ref{sec:separation}. In the absence of prior knowledge, both $\theta_0$ and $\Theta_{0}$ may be treated as tunable parameters and chosen empirically.

\begin{remark}\label{rmk:Kmeans_iterate}
K-means clustering partitions data into $K$ clusters by minimizing within-cluster variance, typically using Lloyd's algorithm~{\rm\cite{lloyd1982least}}, which alternates between assigning points to the nearest cluster center and updating centers as the mean of their assigned points. While effective, Lloyd's algorithm is sensitive to initialization and can converge to suboptimal solutions.

To address this, K-means$^{++}$~{\rm\cite{arthur2006k}} introduces a probabilistic seeding strategy that selects initial centers that are well-separated, thereby significantly enhancing clustering performance. Inspired by this idea, our iterative procedure for determining the number of clusters $\widehat{Q}$ in Algorithm~\ref{alg:Kmeans_unknown_Q} plays a similar role: it uses prior information to adaptively increase the number of clusters, thereby improving separation and robustness. 

The complexity of K-means depends on the clustering at each step and the number of iterations. During each iteration, it requires $\calO(N^2KQ)$ steps to compare the $N(N-1)$rows to the $Q$ cluster centers. The number of K-means iterations can be exponentially large in the worst case. However, using the K-means$^{++}$ initialization can improve the convergence and reduce the number of iterations. In practice, the number of iterations is often small. 
\end{remark}

\subsection{Factorization and post-processing}\label{sec:post_processing}
\subsubsection{Factorizating $\widehat {\bZ}$ into $\widehat{\ba}$ and $\widehat{\bc}$.}
Given the estimated assignment matrix $\widehat {\bKa}$ (hence type matrix $\widehat\bkappa$) and cluster centers $\{\widehat\bv_q\}$, we use a simple linear regression \eqref{eq:LS_eta} to factorize the estimated matrix $\widehat\bZ$ into the graph matrix $\ba$ and the kernel coefficients $\bc$ with negligible computational cost. 

It should be emphasized that there is no guarantee that $\bA$ defined by \eqref{eq:A_eta} is invertible, and regularization is necessary. In fact, the rank and spectrum of $\bA$ carry essential structural information about the underlying system, in particular, which kernels are active or not, as discussed in Remark \ref{rmk:act_inact_kernel}. See Section \ref{sec:theory_factorization} for a detailed discussion on the invertibility of $\bA$ and the associated error analysis. 

By replacing $\bZ$ and $\bKa$ with their estimated $\widehat{\bZ}$ and $\bKa$, we solve \eqref{eq:LS_eta} for $\{\widehat v_q\}$ and recover $\widehat \bc$ via \eqref{eq:cvv}. The graph matrix is then obtained entrywise as
\begin{equation}\label{eq:azv}
	\widehat{a}_{ij} = \frac{\widehat z_{i,j}}{\widehat v_{\kappa_{ij}}}.
\end{equation}
The entire recovery procedure is summarized in Algorithm~\ref{alg:Z_kappa_z_c}. The computational cost of solving the small linear system is negligible compared to the costs of the previous two stages.

\begin{algorithm}[t]
\caption{Factorizing $\widehat\bZ$ into  $\widehat\ba$ and $\widehat\bc$ by using $\widehat\bKa$ and cluster centers.}
\label{alg:Z_kappa_z_c}
\begin{algorithmic}[1] 
\Procedure{Post-processing Decomposition}{$\widehat{\bZ}$, $\widehat\bKa$, $\widehat\kappa$, $\{\widehat\bv_q\}$}
	\State Solve $\widehat\bEta$ from $\widehat\bA \widehat\bEta =\bf1_{N} $ with $\widehat\bA_{i} = [\widehat z_{i, 1}^2, \dots, \widehat z_{i, N}^2]\widehat \bKa_i \in \R^{1\times Q}$ and $\widehat\bA= [\widehat\bA_{i}]_{1\leq i\leq N}\in \R^{N\times Q}$. 
	\State Let $\widehat v_q = \sqrt{1/\widehat\eta_q}$ and let $\widehat{\bc} = 
\begin{bmatrix}
\widehat v_1 \widehat\bv_1^\top, \cdots, \widehat v_Q \widehat\bv_Q^\top
\end{bmatrix}$. 
	\State Construct $\widehat{\ba}=[\widehat{a}_{ij}]$ via $\widehat{a}_{ij}= \widehat z_{i,j}/\widehat v_{\kappa_{ij}}$.
	\State Construct $\widehat{\Phi}_q(x) = \sum_{k = 1}^K \widehat \bc^{(q)}_k \psi_k(x)$ for $q = 1, \dots, \widehat{Q}$. 
	\State \textbf{return} $\widehat{\ba}, \widehat{\bc}, \{\widehat{\Phi}_q\}$.
    \EndProcedure
\end{algorithmic}
\end{algorithm}

\subsubsection{Post-processing via ALS}
To further improve estimation accuracy, we fix the estimated assignment matrix $\widehat{\bKa}$ and perform an additional ALS step to refine the estimates of $\ba$ and $\bc$. Specifically, given fixed $\widehat{\ba}$, we estimate $\bc$ by 
\begin{align}\label{ALS3:aK_to_c}
	\min_{\bc \in \R^{K\times Q}}\calL(\widehat{\ba}, \widehat{\bKa}, \bc) = 
	\frac{1}{LMN}\sum_{l,m,i=1}^{L,M,N}\norm{\dot{X}_{t_l,m}^i -  \innerp{\bB(\bX_{t_l}^m)_i, 
	\text{Diag}(\widehat{\ba}) \widehat{\bKa}_{i} \bc^\top}_F}^2_{\R^d}\,.
\end{align}
Then for given $\widehat{\bKa}, \widehat{\bc}$, we estimate each row of the graph $\ba_{i}$ by
\begin{align}\label{ALS3:Kc_to_a}
	\min_{\ba_i \in [0, 1]^{1 \times(N-1)}}\calL(\ba_{i}, \widehat{\calK}, \widehat{\bc}) = 
	\frac{1}{LM}\sum_{l,m=1}^{L,M}\norm{\dot{X}_{t_l,m}^i -  \innerp{\bB(\bX_{t_l}^m)_i, 
	\text{Diag}(\ba_{i}) \widehat{\bKa}_{i} \widehat{\bc}^\top}_F}^2_{\R^d}\,.
\end{align}
These updates are performed iteratively until convergence, with $\ba$ constrained to satisfy the row normalization condition.
This post-processing step improves robustness to environmental noise or model perturbations; its performance also depends on the RIP property of the sensing operator, though the details are omitted here.
The full procedure is summarized in Algorithm~\ref{alg:ALS-3}.
 
\begin{algorithm}[t]
\caption{Post-processing via ALS}
\label{alg:ALS-3}
\begin{algorithmic}[1] 
\Procedure{Post-Processing ALS}{$\{\bX_{t_l}^m\}_{l = 1, m = 1}^{L, M}$, $\{\psi_k\}_{k = 1}^K$, $\widehat{\bKa}$,  $\varepsilon$, $n_{maxiter}$}
    \State Compute derivative $\{\Delta \bX_{t_l}^m\}$ and sensing tensors $\{\bB(\bX_{t_l}^m)\}_{l, m}$ from basis functions.  
    \State Randomly pick initial value $\widehat{\ba}_0 \in \R^{N\times (N-1)}$, or take $\widehat{\ba}_0$ as the result of Algorithm \ref{alg:Z_kappa_z_c}.
    \For{$\tau = 1, \dots, n_{maxiter}$} 
        \State Solve $\widehat{\bc}_\tau=\argmin_{\bc} \calL(\widehat{\ba}, \widehat{\bKa}, \bc)$ in \eqref{ALS3:aK_to_c} using least square with $\widehat{\ba} = \widehat{\ba}_{\tau - 1}$. 
        \For{$i = 1, \dots , N$} 
        	\State Solve $\widehat{\ba}_{\tau, i}=\argmin_{\ba_i} \calL(\ba_i, \widehat{\bKa}, \widehat{\bc})$ in \eqref{ALS3:Kc_to_a} using least square with $\widehat{\bc} = \widehat{\bc}_{\tau}$. 
        	\State Row-normalize $\ba_\tau$. 
        \EndFor
        \State Construct the matrix $\widehat{\ba}_\tau = [\widehat{\ba}_{\tau, 1}^\top, \dots, \widehat{\ba}_{\tau, N}^\top]^\top$.
        \State Exit loop if $\|\widehat{\ba}_\tau - \widehat{\ba}_{\tau-1}\|_F \leq \varepsilon \|\widehat{\ba}_{\tau-1}\|_F$ and $\|\widehat{\bc}_\tau - \widehat{\bc}_{\tau-1}\|_F \leq \varepsilon \|\widehat{\bc}_{\tau-1}\|_F$.
    \EndFor
	\State Construct $\widehat{\Phi}_q(x) = \sum_{k = 1}^K \widehat \bc^{(q)}_k \psi_k(x)$ for $q = 1, \dots, \widehat{Q}$. 
    \State \textbf{return} $\widehat{\ba}, \widehat {\bc}, \{\widehat{\Phi}_q\}$. 
\EndProcedure
\end{algorithmic}
\end{algorithm}

\subsection{Connections and extensions}\label{sec:alg_extensions}

\subsubsection{Complete network}\label{sec:complete_network}

In the case of a complete network, similar procedures can be employed to recover the assignment matrix $\bKa$ and the kernel coefficients $\bc$. We first apply Algorithm~\ref{alg:ALS} to obtain an estimate $\widehat{\bZ}$. In the second stage, if the number of interaction types $Q$ is known, we directly apply $K$-means clustering to the unnormalized rows of $\widehat{\bZ}$, where the resulting cluster centers correspond to the estimated kernel coefficients. If $Q$ is unknown, we update $\widehat{Q}$ iteratively by comparing the separation metrics $\theta$ and $\Theta$ to their respective thresholds. The assignment matrix $\widehat{\bKa}$ is then constructed via \eqref{eq:Kmeans_Kappa}.

Unlike the general graph setting, where the rows of $\widehat{\bZ}$ lie along $Q$ distinct directions, in the complete network setting, they have exactly $Q$ distinct positions. This allows direct clustering based on Euclidean distance without separating zero and nonzero rows. In practice, this structure leads to improved robustness and accuracy in the clustering outcome. The algorithm is concluded in the Appendix \ref{apdx:complete_networks}. 

More broadly, our framework accommodates not only the fully unknown setting but also cases with partial information. Given the interaction graph, the above complete-network procedure illustrates how one may bypass graph recovery and focus on clustering and kernel estimation. If the type assignments are known, one may directly apply the post-processing ALS refinement to recover the graph and kernel coefficients. Finally, when the kernels are known, the problem reduces to a relatively easier recovery of the graph and type assignments. These special cases highlight the flexibility of our approach and its applicability to a wide range of learning scenarios.

\subsubsection{Integer constraint}
A natural extension of the Algorithm \ref{alg:ALS-3} is to consider a three-way ALS approach. When $\widehat{\bKa}$ is unavailable, we can introduce an additional step to estimate $\bKa$ from given $\widehat{\ba}$ and $\widehat{\bc}$, subject to the constraint $\bKa \in \calB^N$ as defined in \eqref{Def:Bi_state}. However, this is a non-convex integer constraint, which makes direct optimization difficult. To enable efficient least-squares estimation, a convex relaxation of \eqref{Def:Bi_state} can be used. In practice, though, this often leads to poor local minima and suboptimal solutions that do not exhibit the desired structure. Due to this challenge, we omit the constraint on $\bKa$ and instead split the problem into two stages, followed by a post-processing step.

\bigskip

\section{Theoretical Guarantees}\label{Sec:Theory}

\subsection{Exploration measure and space for interaction kernels}

To lay the groundwork for the function space, we first introduce an exploration measure and hypothesis space as in \cite{LZTM19pnas, LMT21_JMLR, LWLM2024}.

We define a function space $L^2(\rho_L)$ for learning the interaction kernel $\bPhi$, where $\rho_L$ is a probability measure that quantifies data exploration to the interaction kernel. 
Denote the pairwise differences 
\begin{equation}\label{Def:r_ij}
	 \br_{ij}(t_l):=X^{j}_{t_l}-X^{i}_{t_l}\quad \text{and}\quad
	 \br_{ij}^m(t_l)
	 :=X^{j,m}_{t_l}-X^{i,m}_{t_l}\,, 
\end{equation}
which are the variables of the interaction kernel. Thus, we define $\rho_L$ as follows.  
\begin{definition}[Exploration measure for interaction kernels]
Given observations of $M$ trajectories at the discrete times $\{t_l\}_{l=0}^{L-1}$, we define an empirical measure and its large-sample limit on $\mathbb{R}^d$ as follows:
\begin{align}
	\rho_{L,M}(d\br) &:= \frac{1}{LM N(N-1)}\sum_{l=0}^{L-1}\sum_{m=1}^{M} \sum_{1\leq i\neq j\leq N} \delta_{\br_{ij}^{m}(t_l)}(d\br), \label{Def:empirical} \\
	 \rho_L(d\br) &:=\frac{1}{L N(N-1)}\sum_{l=0}^{L-1} \sum_{1\leq i\neq j\leq{N}} \E[ \delta_{\br_{ij}(t_l)}(d\br)] 
	 	 \label{Def:exploration}\,,
\end{align} 
where $\delta_{\br}$ denotes the Dirac measure at $\br$, and the notation $\sum_{1\leq i\neq j\leq{N}}$ stands for $\sum_{i=1}^N \sum_{j=1,j\neq i}^{N}$. 
\end{definition}
The empirical measure depends on the sample trajectories, but $\rho_L$ is the large-sample limit, uniquely determined by the distribution of the stochastic process $\bX_{t_0:t_{L-1}}$, and hence data-independent. 

We consider a hypothesis space with basis functions satisfying the following conditions. 
\begin{assumption}[Uniformly bounded basis functions]\label{assum:basis}
The basis functions of the hypothesis space $\calH:=\mathrm{span} \{ \psi_k: \R^d\to \R^d \}_{k=1}^K$ are orthonormal in $L^2(\rho_L)$. Moreover, they are uniformly bounded, i.e., $\sup_{k\in[K]} \|\psi_k\|_\infty \leq L_{\mH}$.
\end{assumption}

The size of the hypothesis space $K$ grows exponentially with the dimension $d$ in nonparametric regression because of the curse of dimensionality. To avoid it, a common practice is to consider radial interaction kernels in the hypothesis space $\calH:=\mathrm{span} \{ \psi_k: \R^+\to \R \}_{k=1}^K$; see, e.g., \cite{LMT21,LZTM19pnas}. We also consider it natural for the interactions to depend only on a small number of variables, possibly unknown (see, e.g., \cite{LearningInteractionVariables_1}).

\subsection{Error bounds for the matrix sensing stage}\label{sec:RIP}
In this section, we provide theoretical guarantees for the first stage based on matrix sensing theory~\cite{GJZ2017, RFP2010, ZSL2017}. For clarity of exposition, we focus on the simplified setting with $d = 1$ and $L = 1$. We analyze the estimation of a single vector $\bZ_i$ for a fixed $i \in [N]$, noting that the overall matrix sensing problem in~\eqref{Def:L_i} decouples into $N$ separate subproblems, one for each $\bZ_i$, $i \in [N]$. 

In the case $d = 1$ and $L = 1$, the sensing operator $\bB_i$ in~\eqref{Def:basis_array} reduces to a sensing matrix $\bB_i(\bX^m_0) \in \R^{(N-1)\times K}$, 
which encodes the relative interaction structure centered on the $N$-th particle. The observed quantity is the velocity of the $i$th particle at time $t = 0$, corrupted by additive Gaussian noise, i.e., 
\begin{equation}\label{eq:bm}
b_m = \dot{X}^{i,m}_0 + \epsilon^{i,m} = \langle \bB_i(\bX^m_0), \bZ_i^\star \rangle_F, \quad \epsilon^{i,m} \sim \mathcal{N}(0, \sigma_\epsilon^2).
\end{equation}
This observation model arises naturally when the system is driven by stochastic forcing and finite-difference approximations are used to estimate derivatives. 
In this case, $\sigma_\epsilon = \sigma \sqrt{\Delta t}$ with $\sigma$ being the diffusion coefficient of \eqref{Eq:ips_K_type} and $\sqrt{\Delta t}$ from the Euler-Maruyama discretization. 
In particular, the loss function $\mathcal{L}_i(\bZ_i)$ in \eqref{Def:L_i} reduces to 
\begin{equation}\label{eq:LossZN}
	\mathcal{L}_i(\bZ_i) = \frac{1}{M}\sum_{m=1}^M\|b_m- \langle \bB_i(\bX^m_0),\bZ_i\rangle_{F}\|^2. 
\end{equation}

We assume the following RIP (restricted isometry property) condition on the sensing matrix. 
\begin{assumption}[RIP condition]\label{conj_RIP}
Fix $0 < \delta < 1$. For any $1 \leq Q \leq ((N-1)\wedge K)$, there exist positive constants $c_0$, $c_1$ depending only on $\delta$ where $c_0 = O(1/\delta^2)$, such that, with probability at least $1 - exp(-c_1 M)$,  whenever $M \geq c_0 Q (N+K)\log(NK)$, the smallest constant $\delta_Q$ satisfying
	\begin{equation}\label{eq:RIP-IPS}
	(1-\delta_{Q}) \norm{Z}_F^2 \leq \frac{1}{CM}\sum_{m=1}^M \langle \bB_i(\bX^m_0),Z\rangle_{F}^2 \leq (1+\delta_{Q}) \norm{Z}_F^2
\end{equation}
for all rank-$Q$ matrix $Z\in \R^{(N-1)\times K}$ and $i\in[N]$, 
is less than $\delta$. 
\end{assumption}

The above RIP condition holds true for nearly isometric random matrices, which include matrices with i.i.d.\ Gaussian entries $\mathcal{N}(0,\frac{1}{M})$, as shown in \cite[Theorem 4.2]{RFP2010}. For general interacting particle systems, the entries of the sensing matrices exhibit nontrivial correlations due to their shared dependence on a common reference particle and the basis functions. While certain techniques from \cite{Liu_NIPS2011} may offer partial insights, verifying RIP-type conditions under such structural dependencies remains a significant challenge. 
Moreover, extending the analysis to include longer time steps and higher dimensions would further complicate the problem. As such, we present the above condition as a plausible foundation for understanding the sample complexity of this structured sensing setup, while acknowledging that its rigorous verification remains an open question. 

This RIP condition yields recovery guarantees for both noiseless and noisy data. The next proposition shows the identifiability of $\bZ_i$ from noiseless data and estimation error bounds from noisy data.  

\begin{proposition}
\label{prop:consequence_RIP}
Let $d=1$, $L=1$, and suppose Assumption~\ref{conj_RIP} holds. Then, for any $1 \leq Q \leq (N - 1) \wedge K$, there exists a constant $C > 0$ such that the following holds with high probability for each $i\in [N]$.  
\begin{itemize}
    \item \textbf{(Noiseless case)}: If $\sigma_\epsilon = 0$ and $M \geq C Q (N + K) \log(NK)$, then $\bZ^*_i$ is the unique minimizer of the objective in \eqref{eq:LossZN}.
    \item \textbf{(Noisy case)}: If $\sigma_\epsilon > 0$ and $M \geq C Q (N + K) \log(NK)$, then any local minimizer $\widehat{\bZ}_i=\widehat{\bu}_i \widehat{\bc}^\top $ of $f(\bu_i,\bc):=\frac{1}{M}\sum_{m = 1}^{M}\norm{b_m - \innerp{\bB_i(\bX_{0}^m), \bu_i\bc^\top}_F}_{\R^d}^2 + \|\bu_i^\top \bu_i-\bc \bc^\top \|^2$ with $b_m$ in \eqref{eq:bm}, 
    satisfies
    \begin{equation}\label{eq:ZN_error}
	    \|\bZ^*_i - \widehat{\bZ}_i \|_F  \lesssim  \sigma_\epsilon \sqrt{\frac{(N + K) Q \log M}{M}} 
    \end{equation}
    for each $i\in [N]$. Here, $A\lesssim B$ means $A\leq C_0 B$ for some generic constant $C_0>0$. 
\end{itemize}
\end{proposition}

The above result follows directly from \cite[Theorem 3.2]{RFP2010} and \cite[Theorem 32]{GJZ2017}, which we present in Section \ref{sec:preRIP}. We note that the observed time steps $L$ and dimension of the particles $d$ do not appear in the stated bound, as the analysis here pertains to the simplified case $d = L = 1$. 

Recall from the formulation in \eqref{Def:L_i} that the overall matrix sensing problem can be decomposed into $N$ subproblems, each corresponding to the estimation of $\bZ_i$ for $i = 1, \dots, N$. Consequently, the same sample complexity and error scaling apply to each $\bZ_i$ individually. Summing over all $i$, we obtain the following global error bound:
\begin{equation}\label{eq:Z_error}
\|\bZ^* - \widehat{\bZ}\|_F  \lesssim \sigma_\epsilon \sqrt{\frac{N(N + K) Q \log M}{M}}.
\end{equation}

\begin{remark}
Note that our algorithm does not estimate the matrices $\bZ_i$ independently: they share a common coefficient matrix $\bc$. Yet, the numerical results in Figure \ref{fig:conv_M} and Figure \ref{fig:conv_N} show that the global errors of the estimator are in the orders $\mathcal{O}(M^{-1/2})$ and $\mathcal{O}(N)$,  agreeing with the above bound. 
\end{remark}

Sample complexity offers another perspective on the RIP condition and the preceding proposition. Numerical experiments in Section \ref{sec:sample-size} show that the required sample size to recover the type matrix $\kappa$ scales approximately linearly in $N+K$. Heuristically, the scaling
$M = O\bigl(Q(N+K)\log(NK)\bigr)$ is natural: a rank-$Q$ matrix has $Q(N+K-Q)$ degrees of freedom, so the $Q(N+K)$ factor, up to logarithmic terms, matches this dimensional count.


\subsection{Separability}\label{sec:separation}
We have shown that, under the RIP condition, the estimator $\widehat{\bZ}$ achieves a small error in the Frobenius norm. We now investigate how this estimation error affects the clustering step, which is sensitive to errors measured in the $(2, \infty)$-norm. For any matrix $A \in \mathbb{R}^{m \times n}$, we define its $(2, \infty)$-norm as the maximal row norm:
\begin{equation}
\|A\|_{2, \infty} := \max_{1 \leq i \leq m} \|A_i\|_{\R^n} = \max_{1 \leq i \leq m} \left( \sum_{j = 1}^n A_{ij}^2 \right)^{1/2},
\end{equation}
where $A_i$ denotes the $i$-th row of $A$. Clearly, $(2, \infty)$-norm is equivalent to the Frobenius norm: 
\begin{equation}\label{lem:2_infty_F}
\|A\|_{2, \infty} \leq \|A\|_F \leq \sqrt{m} \, \|A\|_{2, \infty}.
\end{equation}

Recall that \eqref{eq:ZN_error} holds for each $i\in [N]$. Then, applying \eqref{lem:2_infty_F} to $\|\bZ_i^* - \widehat{\bZ}_i\|_{2,\infty}$, we obtain
\begin{equation}\label{eq:Z_error_2_infty}
\|\bZ^* - \widehat{\bZ}\|_{2,\infty}
= \max_{i=1,\dots,N} \|\bZ_i^* - \widehat{\bZ}_i\|_{2,\infty}
 \lesssim  \sigma_\epsilon \sqrt{\frac{(N+K)Q \log M}{M}},
\end{equation}
which increases at the order of $\sqrt{N}$, favorable over the linear scaling in $N$ of the Frobenius norm $\|\bZ^* - \widehat{\bZ}\|_{F}$ in \eqref{eq:Z_error}. Our numerical experiments in Figure \ref{fig:conv_N} suggest that when the graph weight matrix is sparse, the order $\|\bZ^* - \widehat{\bZ}\|_{2,\infty}$ can be much lower than $\sqrt{N}$.

We show first that if $\|\bZ^* - \widehat{\bZ}\|_{2,\infty}$ is sufficiently small, then the zero and nonzero rows of $\widehat{\bZ}$ can be correctly identified by thresholding in \eqref{eq:I_0}.  
\begin{proposition}[Support recovery]\label{prop:support}
	 	 Recall that $z_{min} = \min_{1 \leq i, j \leq N, a_{ij} \neq 0}\norm{\bZ^*_{i,j}}$ in \eqref{eq:z0}. Suppose
 		\begin{equation}\label{eq:Zest_bd_eps}
 			\|\bZ^* - \widehat{\bZ}\|_{2, \infty}< \varepsilon < \frac{z_{min}}{2}.
 		\end{equation}
 		Then, for each $z_0\in (\varepsilon, z_{min} - \varepsilon)$, we have  
 		$	\|\widehat{\bZ}_{i_1, j_1}\| < z_0 < \|\widehat{\bZ}_{i_2, j_2}\|$ 
 		for any index pairs $(i_1, j_1), (i_2, j_2)$ such that $a^*_{i_1 j_1} = 0, a^*_{i_2 j_2} \neq 0$. 
\end{proposition}
\begin{proof}
	Since $a^*_{i_1, j_1} = 0$, the true matrix $\bZ^*$ satisfies $\|\bZ^*_{i_1, j_1}\| = 0$, so that 
	$$\|\widehat{\bZ}_{i_1, j_1}\| \leq \|{\bZ}^*_{i_1, j_1}\| + \|{\bZ}^*_{i_1, j_1} - \widehat{\bZ}_{i_1, j_1}\| < \varepsilon.$$
	where the second inequality follows from \eqref{eq:Zest_bd_eps} and the definition of $(2, \infty)$-norm. Moreover,  since $a^*_{i_2, j_2}>0$, we have $\|\bZ^*_{i_2, j_2}\| \geq z_{min}$ and 
	$$\|\widehat{\bZ}_{i_2, j_2}\| \geq \|{\bZ}^*_{i_2, j_2}\| - \|{\bZ}^*_{i_2, j_2} - \widehat{\bZ}_{i_2, j_2}\| > z_{min} - \varepsilon.$$
	Therefore,  with $\varepsilon < \frac{z_{min}}{2}$, for any $z_0 \in (\varepsilon, z_{min} - \varepsilon)$, we have $ \|\widehat{\bZ}_{i_1, j_1}\| < z_0 < \|\widehat{\bZ}_{i_2, j_2}\|$. 
\end{proof}

In practice, when $z_{min}$ is not available, we tune the parameter $z_0$ to be in range $(\varepsilon, z_{min} - \varepsilon)$ with $\varepsilon= O( \sigma_\epsilon \sqrt{\frac{(N + K) Q \log M}{M}} )$ according to Proposition \ref{prop:consequence_RIP}.

Next, we show that if $\varepsilon$ is small enough, the result of a successful Algorithm \ref{alg:Kmeans_unknown_Q} is guaranteed to provide the correct classification of the rows of $\widehat{\bZ}$. Correct classification depends on two criteria: (i) elements from different true clusters must be assigned to different estimated clusters, and (ii) elements belonging to the same true cluster must be assigned to the same estimated cluster. In particular, this implies that the estimated number of clusters, $\widehat{Q}$, matches the true number.

Recall that if $\{I_q, \bv_q\}_{q = 1}^{\widehat{Q}}$ is the successful output of Algorithm \ref{alg:Kmeans_unknown_Q}, then it must satisfy
\begin{equation}\label{eq:theta_Theta_proof}
	\Theta \leq \Theta_0 , \quad  \theta \geq \theta_0,
\end{equation} 
where $\theta := \min_{1 \leq p < q \leq \widehat{Q}} \arccos \frac{\innerp{\widehat\bv_p, \widehat\bv_q}}{\|\widehat\bv_p\| \|\widehat\bv_q\|}$
and $\Theta = \max_{1 \leq q \leq \widehat{Q}}\,\,\max_{(i, j), (i', j') \in I_q} \,\arccos \innerp{\overline{\widehat\bZ}_{i, j},\overline{\widehat\bZ}_{i', j'}}$ 
are defined in \eqref{eq:D_max} and \eqref{eq:theta}, respectively. Here, $\theta_0$ and $\Theta_0$ are the prescribed separation thresholds. Intuitively, $\Theta\le\Theta_0$ controls the within-cluster dispersion while $\theta\ge\theta_0$ guarantees sufficient between-cluster separation. 
\begin{proposition}[Cluster separability]\label{prop:separate}
	Let $\{I_q, \bv_q\}_{q = 1}^{\widehat{Q}}$ be the clustering results of Algorithm \ref{alg:Kmeans_unknown_Q} satisfying \eqref{eq:theta_Theta_proof}. 
	 Suppose $\|\bZ^* - \widehat{\bZ}\|_{2, \infty} < \varepsilon$ with $\varepsilon$ satisfying
	\begin{equation}\label{eq:eps_theta_0}
		\frac{\pi\varepsilon}{(z_{min} - \varepsilon)} \leq \frac{\theta_{min}}{2},
	\end{equation}
	where $\theta_{min}$ and $z_{min}$ are defined in \eqref{eq:theta0} and \eqref{eq:z0}. 
	Then for any
	\begin{equation}\label{eq:bd_Theta0_theta0} 
		\Theta_0 \in \left(\frac{\pi\varepsilon}{(z_{min} - \varepsilon)}, \theta_{min} - \frac{\pi\varepsilon}{(z_{min} - \varepsilon)}\right), \quad \theta_0 \in \left(\Theta_0+\frac{\pi\varepsilon}{(z_{min} - \varepsilon)}, \theta_{min}\right), 
	\end{equation}
	and any index pairs $(i_2, j_2), (i_3, j_3)$ with $a^*_{i_2, j_2}, a^*_{i_3, j_3} \neq 0$, the following holds.
	\begin{itemize}
		\item[(i)] Inter-cluster consistency: 
		\begin{equation}\label{eq:inter_cluster_consistency}
			\kappa^*_{i_2, j_2} \neq \kappa^*_{i_3, j_3} \quad \implies  \quad(i_2, j_2)\in I_{q_2},  \ (i_3, j_3) \in I_{q_3}, \ \text{ for }\  q_2 \neq q_3.
		\end{equation}
		\item[(ii)] Intra-cluster consistency:
		\begin{equation}\label{eq:intra_cluster_consistency}
			\kappa^*_{i_2, j_2} = \kappa^*_{i_3, j_3} \quad \implies  \quad(i_2, j_2), (i_3, j_3) \in I_q, \text{ for some } q = 1, \dots, \widehat{Q}.
		\end{equation}
	\end{itemize}
\end{proposition}
\begin{proof}
We first prove the inter-cluster consistency by contradiction. Suppose $\kappa^*_{i_2, j_2} \neq \kappa^*_{i_3, j_3} $ but $(i_2, j_2), (i_3, j_3)$ are classified in the same cluster with center $\bv_q$. Then, the intra-cluster angle satisfies
\[
	\arccos\innerp{\overline{\widehat{\bZ}}_{i_2, j_2}, \overline{\widehat{\bZ}}_{i_3, j_3}} \leq \Theta_0.
\]
From Lemma \ref{lem:geom}, it holds that
	\begin{equation}\label{eq:Z_i2j2_bar}
		\|\overline{{\bZ}}^*_{i_2, j_2} - \overline{\widehat{\bZ}}_{i_2, j_2}\| \leq \frac{1}{z_{min} - \varepsilon} \| {{\bZ}}^*_{i_2, j_2} - {\widehat{\bZ}}_{i_2, j_2}\| < \frac{\varepsilon}{z_{min} - \varepsilon},
	\end{equation}
	since $\| \bZ^*_{i_2, j_2}\| \geq z_{min}$ and $\|\widehat{\bZ}_{i_2, j_2}\| \geq \| \bZ^*_{i_2, j_2}\| - \| \bZ^*_{i_2, j_2} - \widehat{\bZ}_{i_2, j_2}\| > z_{min} - \varepsilon$. Together with \eqref{eq:trig_id}, \eqref{eq:Z_i2j2_bar} implies that
\begin{equation}
	\arccos\innerp{\overline{{\bZ}}^*_{i_2, j_2}, \overline{\widehat{\bZ}}_{i_2, j_2}} = 2 \arcsin \frac{\|\overline{{\bZ}}^*_{i_2, j_2} - \overline{\widehat{\bZ}}_{i_2, j_2}\| }{2}
	\leq
	\frac{\pi\varepsilon}{2(z_{min} - \varepsilon)}. 
\end{equation}
Note that this estimate also works for the index pair $(i_3, j_3)$. Then, 
\begin{equation*}
	\arccos\innerp{\overline{{\bZ}}^*_{i_2, j_2}, \overline{{\bZ}}^*_{i_3, j_3}} \leq 
	\arccos\innerp{\overline{{\bZ}}^*_{i_2, j_2}, \overline{\widehat{\bZ}}_{i_2, j_2}} + 
	\arccos\innerp{\overline{\widehat{\bZ}}_{i_2, j_2}, \overline{\widehat{\bZ}}_{i_3, j_3}}+
	\arccos\innerp{\overline{{\bZ}}^*_{i_3, j_3}, \overline{\widehat{\bZ}}_{i_3, j_3}}
	\leq  
	\Theta_0 + \frac{\pi\varepsilon}{(z_{min} - \varepsilon)}.
\end{equation*}
However, by the upper bound in our selection of $\Theta_0$ in \eqref{eq:bd_Theta0_theta0} and the condition  \eqref{eq:eps_theta_0}, we have 
\begin{equation}
	\arccos\innerp{\overline{{\bZ}}^*_{i_2, j_2}, \overline{{\bZ}}^*_{i_3, j_3}} < \theta_{min},
\end{equation}
which contradict to the definition of $\theta_{min}$ in \eqref{eq:theta0}. Thus, $(i_2, j_2), (i_3, j_3)$ must be classified in different clusters.

	Next, we prove the intra-cluster consistency through contradiction arguments. Suppose $\kappa^*_{i_2, j_2} = \kappa^*_{i_3, j_3}$ but $\widehat{\bZ}_{i_2, j_2}$ and $\widehat{\bZ}_{i_3, j_3}$ are classified to cluster centers $\bv_{q_2}$ and $\bv_{q_3}$ respectively. From the intra-cluster angle maximum $\Theta_0$, we have 
\begin{equation}\label{eq:angleZ2Z3}
	\arccos\innerp{\overline{\widehat{\bZ}}_{i_2, j_2}, \overline\bv_{q_2}} \leq \Theta_0/2, 	\quad \arccos\innerp{\overline{\widehat{\bZ}}_{i_3, j_3}, \overline\bv_{q_2}} \leq \Theta_0/2,
\end{equation}
since $\bv_{q_2}$ and $\bv_{q_3}$ are the cluster centers. 
Note that \eqref{eq:Z_i2j2_bar} holds for $(i_2, j_2)$ and $(i_3, j_3)$. Moreover, $\overline{\bZ}^*_{i_2, j_2} = \overline{\bZ}^*_{i_3, j_3}$  since $\kappa^*_{i_2, j_2} = \kappa^*_{i_3, j_3}$. Hence, 
\begin{equation}
	\| \overline{\widehat{\bZ}}_{i_2, j_2} - \overline{\widehat{\bZ}}_{i_3, j_3}\| \leq \| \overline{\widehat{\bZ}}_{i_2, j_2} -\overline{\bZ}^*_{i_2, j_2}\| + \| \overline{\bZ}^*_{i_3, j_3}-  \overline{\widehat{\bZ}}_{i_3, j_3}\|  < \frac{2\varepsilon}{z_{min} - \varepsilon}.
\end{equation}
Then, applying \eqref{eq:trig_id} again, we obtain 
\[
	\arccos\innerp{\overline{\widehat{\bZ}}_{i_2, j_2}, \overline{\widehat{\bZ}}_{i_3, j_3}} = 2 \arcsin \frac{\| \overline{\widehat{\bZ}}_{i_2, j_2} - \overline{\widehat{\bZ}}_{i_3, j_3}\| }{2} \leq \frac{\pi\varepsilon}{z_{min} - \varepsilon}.
\]
Together with \eqref{eq:angleZ2Z3}, it holds
\begin{equation}
	\arccos\innerp{\overline\bv_{q_2}, \overline\bv_{q_3}} 
	\leq \arccos\innerp{\overline{\widehat{\bZ}}_{i_2, j_2}, \overline\bv_{q_2}} + \arccos\innerp{\overline{\widehat{\bZ}}_{i_3, j_3}, \overline\bv_{q_2}} +  \arccos\innerp{\overline{\widehat{\bZ}}_{i_2, j_2}, \overline{\widehat{\bZ}}_{i_3, j_3}}
	\leq \Theta_0 + \frac{\pi \varepsilon	}{z_{min} - \varepsilon}.
\end{equation}
However, since we require that the inter-cluster angle $\theta$ to be larger than $\theta_0$, so that 
\begin{equation}
	\arccos \innerp{\overline\bv_{q_2}, \overline\bv_{q_3}} \geq \theta_0.
\end{equation}
We reach a contradiction from the lower bound of $\theta_0$ in \eqref{eq:bd_Theta0_theta0}. Thus, $\widehat{\bZ}_{i_2, j_2}$ and $\widehat{\bZ}_{i_3, j_3}$ must be classified to a same center.
\end{proof}

Our choice of the lower bound of $\Theta_0$ in \eqref{eq:bd_Theta0_theta0}  guarantees the condition \eqref{eq:theta_Theta_proof}, since the intra-cluster angles have an upper bound $ \frac{\pi \varepsilon	}{z_{min} - \varepsilon}$ when $\|\bZ^* - \widehat{\bZ}\|_{2, \infty} < \varepsilon$. We set $\theta_0$ to be less than $\theta_{min}$ so that the algorithm can separate the clusters with centers achieving the minimum angle $\theta_{min}$. Also note that the range of $\Theta_0$ enables the selection of $\theta_0$. 
	
\begin{remark}[Selection of $\Theta_0$ and $\theta_0$]
	The admissible ranges of $\Theta_0$ and $\theta_0$ depend on the $\bZ$ estimation error $\varepsilon$ and oracle parameters $z_{min}$ and $\theta_{min}$. In particular, smaller $\varepsilon$ and larger $z_{min}, \theta_{min}$ allow for more favorable choices of these thresholds
\end{remark}

Condition \eqref{eq:eps_theta_0} holds when $\varepsilon \leq \frac{z_{min}\theta_{min}}{2\pi}$, and Condition \eqref{eq:Zest_bd_eps} requires $\varepsilon \leq \frac{z_{min}}{2}$. 
Therefore, the conditions in Proposition \ref{prop:support} and Proposition \ref{prop:separate} are satisfied only if 
\begin{equation}\label{eq:eps_0}
	\|\bZ^* - \widehat{\bZ}\|_{2, \infty} < \min\left(\frac{z_{min}\theta_{min}}{2\pi}, \frac{z_{min}}{2}\right) := \varepsilon_0.
\end{equation}
According to \eqref{eq:Z_error_2_infty} and the conditions in Proposition \ref{prop:consequence_RIP}, we can achieve \eqref{eq:eps_0} by increasing the sample size to the order $ M = \mathcal{O}\left(\frac{(N+K)Q \log(NK)}{(\varepsilon_0 / \sigma_\epsilon)^2}\right)$.

\begin{remark}
The condition $\|\bZ^* - \widehat{\bZ}\|_{2,\infty} < \varepsilon_0$ does not guarantee successful clustering via K-means, particularly when some clusters are small. Nevertheless, it guarantees the intra- and inter-cluster consistency. 
Algorithm~\ref{alg:Kmeans_unknown_Q} incorporates iterative refinement using K-means, which increases the probability of convergence as discussed in Remark~\ref{rmk:Kmeans_iterate}.

Although various analyses of the success probability of K-means exist (e.g.,~\cite{ChenYang2021,LuZhou2016,ZCYZ2023}), they often depend on specific distributional assumptions. In contrast, we assume no prior distribution on the graph or kernel coefficients, and the distribution of the matrix $\widehat{\bZ}$ seems intractable. We therefore adopt a deterministic condition which, though conservative, is practical and appears justified by our numerical experiments.
\end{remark}

Note that Proposition \ref{prop:separate} does not necessarily imply that $\widehat \bkappa = \bkappa^*$, because of the unidentifiability of $\kappa_{ij}$ when $a_{ij} = 0$ (recall Remark \ref{rmk:act_inact_kernel}). We show next that we can identify those $\kappa_{ij}$ with $ a_{ij}>0$, which we call active kernels.  
\begin{definition}[Active kernel indices]
	For a given graph matrix $\ba^*$ and type matrix $\bkappa^*$, we define the set of active kernel indices as
	\begin{equation}\label{eq:Q_act}
		\mathcal{Q}^*_{\rm act}:=\bigl\{q = 1, \dots, Q:\exists\,(i,j),\, a^*_{ij}\neq 0 \text{ and } \kappa^*_{ij}=q\bigr\}.
	\end{equation}	
\end{definition}
We obtain the following result, which characterizes the estimation error of the type matrix $\bkappa$.
\begin{proposition}

Suppose $\|\bZ^* - \widehat \bZ\|_{2, \infty} < \varepsilon_0$ with $\varepsilon_0$ is defined in \eqref{eq:eps_0}. Then Algorithm \ref{alg:Kmeans_unknown_Q} provides $\widehat \bkappa$ and $\widehat{Q}$ such that 
	\begin{equation}\label{eq:kappa_recovery}
		\quad \widehat Q = |\calQ^*_{act}|, \quad \max_{(i, j): a_{ij}\neq 0}|\kappa_{ij}^* - \widehat \kappa_{ij}| = 0.
	\end{equation}
\end{proposition}
\begin{proof}
Note that $\|\bZ^* - \widehat \bZ\|_{2, \infty} < \varepsilon_0$ implies the correct support recovery by Proposition \ref{prop:support}, also the inter-cluster consistency \eqref{eq:inter_cluster_consistency}  and intra-cluster consistency \eqref{eq:intra_cluster_consistency} by Proposition \ref{prop:separate}.

Let $E^*_{\mathrm{act}}:=\{(i,j): a^*_{ij}\neq 0\}$ denote the set of active index pairs.
For $q\in[Q]$, define the true type classes
\[
E_q^* \;:=\; \{(i,j)\in E^*_{\mathrm{act}}:\ \kappa^*_{ij}=q\}.
\]
Let $\{I_1,\dots,I_{\widehat Q}\}$ be the (nonempty) clusters produced by the Algorithm \ref{alg:Kmeans_unknown_Q}, so that
$\bigsqcup_{r=1}^{\widehat Q} I_r = E^*_{\mathrm{act}}$.

Fix any $q$ with $E_q^*\neq\emptyset$ (i.e., $q\in\mathcal Q^*_{\mathrm{act}}$). 
By intra-cluster consistency \eqref{eq:intra_cluster_consistency}, all pairs in $E_q^*$ belong to a \emph{single} estimated cluster; denote it by $I_{r(q)}$ for some $r(q)\in\{1,\dots,\widehat Q\}$. 
By inter-cluster consistency \eqref{eq:inter_cluster_consistency}, no pair with a different true type can belong to $I_{r(q)}$, hence
\[
I_{r(q)} \;=\; E_q^* .
\]
Therefore the map $r:\mathcal Q_{\mathrm{act}}\to\{1,\dots,\widehat Q\}$ is well-defined.

\emph{(Injectivity).}
If $q_1\neq q_2$ are both active, then $E_{q_1}^*$ and $E_{q_2}^*$ are disjoint and, by the inter-cluster condition, must occupy distinct estimated clusters; hence $r(q_1)\neq r(q_2)$. Thus $r$ is injective.

\emph{(Surjectivity).}
Take any estimated cluster $I_r$. By construction, $I_r\subset E^*_{\mathrm{act}}$ and is nonempty, so its elements all have some true type $q\in[Q]$. The intra-cluster consistency forces \emph{all} elements of $I_r$ to have that same true type $q$, so $I_r=E_q^*$ and $q\in\mathcal Q^*_{\mathrm{act}}$. Hence $r$ is surjective onto $\{1,\dots,\widehat Q\}$.

Combining injectivity and surjectivity, $r$ is a bijection between $\mathcal Q^*_{\mathrm{act}}$ and $\{1,\dots,\widehat Q\}$, so
\[
\widehat Q \;=\; |\mathcal Q^*_{\mathrm{act}}|.
\]

Finally, define the estimated type on active pairs by $\widehat\kappa_{ij}=r(\kappa^*_{ij})$ for $(i,j)\in E^*_{\mathrm{act}}$.
By the above, $(i,j)\in E_q^*$ iff $(i,j)\in I_{r(q)}$, so $\widehat\kappa_{ij}$ equals $\kappa^*_{ij}$ \emph{up to the permutation} $r$ of the labels. Using the ordering convention introduced in Remark \ref{rmk:cluster_order}, we obtain
\[
\widehat\kappa_{ij}=\kappa^*_{ij}\quad\text{for all }(i,j)\in E^*_{\mathrm{act}},
\]
which yields
$
\max_{(i,j):\,a_{ij}\neq 0}|\kappa_{ij}^* - \widehat \kappa_{ij}| = 0.
$
\end{proof}

\subsection{Factorization}\label{sec:theory_factorization}
We now characterize the error in the factorization step of Algorithm \ref{alg:Z_kappa_z_c} based on the estimation error $\| \bZ^* - \widehat\bZ\|_{2, \infty}$. We will need the following geometry fact that relates to the cluster mean of cosine distance, and its proof is postponed to Appendix \ref{apdx:geom}.

\begin{lemma}[Spherical mean under cosine $k$-means]\label{lem:geom_avg}
Let $\mathbf{v}^*, \mathbf{z}_1,\dots,\mathbf{z}_n \in \mathbb{R}^d$ satisfy $\|\mathbf{v}^*\| = \|\mathbf{z}_i\| = 1$ and $
\|\mathbf{v}^* - \mathbf{z}_i\| \le \varepsilon$ for all $i=1,\dots,n$ with $\varepsilon < \sqrt{2}$. Define the (Euclidean) mean
$
\mathbf{u} := \frac{1}{n} \sum_{i=1}^n \mathbf{z}_i
$
and the cosine-$k$-means centroid
$
\widehat\bv := \frac{\mathbf{u}}{\|\mathbf{u}\|}\,.
$
Then
\[
\|\mathbf{v}^* - \widehat\bv\| \;\le\; \varepsilon.
\]
\end{lemma}

The following theorem shows that the factorization from $\mathbf{Z}$ to $(\mathbf{a},\mathbf{c})$ is stable. 
\begin{theorem}[Stability of factorization from $\mathbf{Z}$ to $(\mathbf{a},\mathbf{c})$]\label{thm:postproc-stability}
Suppose $\|\bZ^* - \widehat \bZ\|_{2, \infty}\ < \varepsilon$ with $\varepsilon < \sqrt{2}$. Recall that $z_{i, j}^*$ and $\widehat z_{i, j}$ denote the row-norms of $\bZ^*$ and $\widehat\bZ$ in \eqref{eq:z_norm}. For each kernel type $q\in[Q]$, we let $v^*_q:=\|\mathbf{c}^{*(q)}\|>0$ and $\eta^*_q:=1/v^{*2}_q$. Let $v_{\min}:=\min_q \{v^*_q, \widehat v_q\}$, $v_{\max}:=\max_q \{v^*_q, \widehat v_q\}$. 

Assume correct assignments, i.e. \eqref{eq:kappa_recovery} holds, and the matrix $\mathbf{A}^*\in\mathbb{R}^{N\times Q}$ defined by \eqref{eq:A_eta} with the true parameters satisfies $\sigma_{\min}(\mathbf{A}^*)>0$.
Let $\widehat{\mathbf{A}}$ be constructed from $\widehat{\mathbf{Z}}$ analogously to $\mathbf{A}^*$, and let $\widehat{\boldsymbol{\eta}}$ be the least-squares solution of 
$
\widehat{\mathbf{A}}\,\widehat{\boldsymbol{\eta}}=\mathbf{1}_N.
$
Then, whenever
\begin{equation}\label{eq:eps_factorization}
	\varepsilon < \min\left(\frac{\sigma_{min}(\bA^*)}{3 N v_{max}}, \ v_{max}, \ (2-\sqrt 2)z_{min}\right),
\end{equation}
the following bounds hold:
\begin{align}
\|\boldsymbol{\eta}^* - \widehat{\boldsymbol{\eta}}\|
\;\;\;\;\;&\le\;
\frac{\sqrt{Q} N(2v_{max} + \varepsilon)\varepsilon}
{v_{min}^2[\sigma_{\min}(\mathbf{A}^*) - N(2v_{max} + \varepsilon)\varepsilon]}\,,
\label{eq:eta-bound}
\\
\max_{q = 1, \dots, Q}\|\Phi_q^* - \widehat\Phi_q\|_{L^2(\rho)}
\;&\le\; {\frac{v_{\max}^3}{2 v_{min}^2}\,\frac{\sqrt{Q} N(2v_{max} + \varepsilon)\varepsilon}
{\sigma_{\min}(\mathbf{A}^*) - N(2v_{max} + \varepsilon)\varepsilon}+ \frac{ v_{max} \,\varepsilon}{z_{min} - \varepsilon}\,}\label{eq:c-bound}
\\
\|\mathbf{a}^* - \widehat \ba\|_F
\;\;\;\;\;&\le\;{\frac{\varepsilon}{v_{min}} + \frac{v_{max}^4}{2v_{min}^2}\frac{\sqrt{Q} N^{\frac 32}(2v_{max} + \varepsilon)\varepsilon}{\sigma_{\min}(\mathbf{A}^*) - N(2v_{max} + \varepsilon)\varepsilon}.}
\label{eq:a-bound}
\end{align}
\end{theorem}
In particular, when $\varepsilon$ is small (so that $\varepsilon^2$ is negligible), Theorem \ref{thm:postproc-stability} implies that
\begin{equation}\label{eq:a_c_error_approx}
\max_{q = 1, \dots, Q}\|\Phi_q^* - \widehat\Phi_q\|_{L^2(\rho)} \lesssim \left(
\frac{N\sqrt{Q} v_{max}^4}{\sigma_{min}(\bA^*)v^2_{min}} + \frac{v_{max}}{z_{min}}\right)\,\varepsilon,
\qquad
\|\mathbf{a}^* - \widehat{\mathbf{a}}\|_F
\;\lesssim\;
\Big(\frac{N\sqrt{NQ} v_{\max}^5}{\sigma_{\min}(\mathbf{A^*}) v_{\min}^4}+ \frac{1}{v_{\min}}\Big)\,\varepsilon.
\end{equation}

\begin{proof}
(i). We first control the error in estimating $\bEta$. Recall that $\bA^* \bEta^* = \widehat \bA \widehat \bEta =  \mathbf{1}_N$. Let $\bA^{*\dagger}$ and $\widehat \bA ^\dagger$ denote the pseudo-inverse of $\bA^*$ and $\widehat \bA$. By the assumption that $\sigma_{min}(\bA^*) >0$, if $\|\bA^* - \widehat \bA\|_2 < \sigma_{min}(\bA^*)$, it holds that $\bA^{*\dagger}\bA^* = \widehat \bA^\dagger \widehat \bA = \Id_Q$, where $\Id_Q$ is the identity matrix in $\R^{Q\times Q}$. Therefore 
\begin{equation*}
	\widehat \bEta - \bEta^* = (\widehat \bA^\dagger \bA^*- \widehat \bA^\dagger \widehat \bA)\bEta^* = \widehat \bA^\dagger(\bA^* - \widehat \bA)\bEta^*.
\end{equation*}
Hence we have 
\begin{equation}\label{eq:eta_error}
	\|\bEta^* - \widehat \bEta\| \leq \|\widehat \bA^\dagger\|_2\|\bA^* - \widehat \bA\|_{2}\|\bEta^*\| \leq \frac{\|\mathbf{A}^* - \widehat{\mathbf{A}}\|_2\,\|\boldsymbol{\eta}^*\|}{\sigma_{\min}(\mathbf{A}^*) - \|\mathbf{A}^* - \widehat{\mathbf{A}}\|_2},
	\end{equation}
where the second inequality follows from the fact that $\|\widehat \bA^\dagger\|_2 = 1/{\sigma_{min}(\widehat \bA)}$ and $\sigma_{min}(\widehat \bA) \geq \sigma_{min}(\bA^*) - \|\bA^* - \widehat \bA\|_2$ since singular values are 1-Lipschitz continuous with respect to the operator norm. 
 It is clear that $\|\bEta^*\| \leq \frac{\sqrt{Q}}{v_{min}^2}$ by contruction, we are left to control $\|\mathbf{A}^* - \widehat{\mathbf{A}}\|_2$.

Since $\widehat \bkappa = \bkappa^*$, it holds that $\widehat{\bKa} = \bKa^*$ and therefore the $i$-th row holds that 
\begin{equation}
	 \bA^*_i - \widehat \bA_i = [z_{i, 1}^{* \,2} - \widehat z_{i, 1}^2, \dots, z_{i, N}^{* \,2} - \widehat z_{i, N}^2]\bKa_i^*.
\end{equation}
Take Euclidean norm $\|\cdot \|$ in $\R^{1 \times Q}$, for a fixed $i$, we have
\begin{equation}
	\|\bA^*_i - \widehat \bA_i\| \leq \|\bKa_i^*\|_2(\max_j{z_{i, j}^*} + \max_j{\widehat z_{i, j}})\sqrt{\sum_{j \neq i, j = 1}^{N}\Delta z_{i, j}^2},
\end{equation}
where $\|\bKa_i^*\|_2$ represents the operator 2-norm. Now for each $(i,j)$, 
\[
|\Delta z_{i,j}|
=
\big|\|\mathbf{Z}_{i, j}^*\| - \|\widehat{\mathbf{Z}}_{i, j}\|\big|
\le \|\mathbf{Z}_{i, j}^* - \widehat{\mathbf{Z}}_{i, j}\| \leq \|\mathbf{Z}_{i, j}^* - \widehat{\mathbf{Z}}_{i, j}\|_{2, \infty} < \varepsilon,
\]
and because of the row-normalization constraints on $\ba^*$,  
$$\max_{i, j}{z^*_{i, j}} = \max_{i, j} a^*_{ij} v^*_{\kappa^*_{ij}} \leq v_{max}.$$
Therefore
$
	\max_{i, j}{\widehat z_{i, j}} + \max_{i, j}{z_{i, j}^*}< 2v_{max} + \varepsilon. $
At last, by the construction of $\bKa_i^*\in \{0, 1\}^{(N-1)\times Q}$ in \eqref{Def:Assign_bKa}, it has exactly one 1 for each row, hence 
\begin{equation}
	\|\bKa_i^*\|_2 \leq \|\bKa_i^*\|_F = \sqrt{N-1}.
\end{equation}
We then conclude that 
\begin{equation}
	\|\bA^*_i - \widehat \bA_i\| \leq \sqrt{N-1}(2v_{max} + \varepsilon)\varepsilon,
\end{equation}
Therefore, it holds
\begin{equation}
	\|\bA^* - \widehat \bA\|_2 \leq N(2v_{max} + \varepsilon)\varepsilon, 
\end{equation}
substituting the estimate into \eqref{eq:eta_error} yields \eqref{eq:eta-bound}. The requirement $\|\bA^* - \widehat \bA\|_2 < \sigma_{min}(\bA^*)$ is satisfies by the upper bound of $\varepsilon$ in \eqref{eq:eps_factorization}.

(ii). We then consider the estimation error of the kernels. Since we choose an orthonormal basis $\{\psi_k\}$ in $L^2(\rho)$, it suffices to consider the bound for coefficients. 
Suppose $(i, j) \in I_q$, we have $\overline \bZ^*_{i, j} = \bv^*_q$, where $\bv^*_q$ is the normalization of $\bc^{*(q)}$ so that $\bc^{*(q)} = v_q^* \bv^*_q,$ with $v^*_q$ represents the 2-norm of $\bc^{*(q)}$. 
We first estimate the error in $\bv_q$. According to \eqref{eq:Z_i2j2_bar}, we have for any $(i, j) \notin I_0$, 
$\|\overline \bZ^*_{i, j} -  \overline {\widehat \bZ}_{i, j}\| < \frac{\varepsilon}{z_{min} - \varepsilon}$. Moreover, by \eqref{cos_dist_center}, the estimated cluster centers derived from the K-means algorithm satisfy
\begin{equation}
	\widehat\bu_q = \frac{1}{I_q}\sum_{(i, j) \in I_q} \overline{\widehat{\bZ}}_{i, j}, \quad \widehat\bv_q = \frac{\widehat\bu_q}{\|\widehat\bu_q\|}, \quad q = 1, \dots, Q.
\end{equation}
Then by Lemma \ref{lem:geom_avg}, we have 
\begin{equation}
	\|\bv_q^* - \widehat \bv_q\| < \frac{\varepsilon}{z_{min} - \varepsilon},
\end{equation}
provided $\frac{\varepsilon}{z_{min} - \varepsilon} < \sqrt{2}$, which was achieved if $\varepsilon$ satisfies \eqref{eq:eps_factorization}.
We then estimate the error in $v_q$. The scalar map $g(t)=1/t^2$ satisfies $|g'(t)|=2 t^{-3}$. Hence, by the Mean-Value Theorem, 
$
|\eta_q^* - \widehat \eta_q| = |1/v^{*2}_q - 1/\widehat{v}_q^2| = 2 t^{-3} |v_q^* - \widehat v_q|,
$
for some $t$ between $\widehat{v_q}$ and $v_q^*$. Hence, for each $q$, 
\begin{equation}\label{eq:v-v}
	|v_q^* - \widehat v_q| = \frac{t^3}{2} |\eta_q^* - \widehat \eta_q| \leq \frac{v_{max}^3}{2}|\eta_q^* - \widehat \eta_q| \leq \frac{v_{max}^3}{2} \|\bEta^* - \widehat\bEta\|.
\end{equation}
At last, since $\|\bv^*\| = 1$, it holds that, 
\[
\|\mathbf{c}^{*(q)} - \widehat{\mathbf{c}}^{(q)}\|
=
|v^*_q -  \widehat v_q| \|\bv^*\|  + |\widehat v_q|\|\bv^*_q - \widehat \bv_q\|
\leq \frac{v_{max}^3}{2}\|\bEta^* - \widehat\bEta\| + \frac{ v_{max} \,\varepsilon}{z_{min} - \varepsilon}\,. 
\]
Indeed, \eqref{eq:c-bound} follows by substituting \eqref{eq:eta-bound} into the equation above. 

(iii). Finally, we estimate the error in inferring $\ba$. If $a^*_{ij} = 0$, we can recover $\widehat{a}_{ij} =0$ because of Proposition \ref{prop:support}. Then if $a^*_{ij} \neq 0$, recall that $\ba$ is recovered entrywise through \eqref{eq:azv}, it holds 
\begin{align}
	\| \ba^* - \widehat \ba\|_F 
	&= 
	\sqrt{\sum_{i \neq j} \abs{\frac{z^*_{i, j}}{v^*_{\kappa^*_{ij}}} - \frac{\widehat z_{i, j}}{\widehat v_{\widehat \kappa_{ij}}}}^2} 
	\leq 
	\left( \max_{i \neq  j}\frac{1}{\widehat v_{\widehat \kappa_{ij}}}\right)
	\sqrt{\sum_{i \neq j} \abs{z^*_{i, j} - \widehat z_{i, j}}^2} 
	+
	\left( \max_{ i \neq j}\abs{\frac{1}{v^*_{\kappa^*_{ij}}} - \frac{1}{\widehat v_{\widehat \kappa_{ij}}}}\right)\sqrt{\sum_{i \neq j} \abs{ z^*_{i, j}}^2 }\notag\\
	&\leq
	\frac{1}{v_{min}}\|\bZ^* - \widehat \bZ\|_F + \frac{1}{v_{min}^2}\left(\max_{q}\abs{v^*_{q} - \widehat v_{q}}\right)\| \bZ^*\|_F, \label{eq:a-a2}
\end{align}
where the first step is the triangle inequality and the second follows from the definitions and a similar Mean-Value Theorem for $f(t) = 1/t^2$. Moreover, 
\begin{equation}
	\|\bZ^*\|_F  = \sqrt{\sum_{i \neq j}\abs{a^*_{ij} v^*_{\kappa^*_{ij}}}^2} \leq v_{max}\sqrt{\sum_{i \neq j}\abs{a^*_{ij}}^2} = v_{max}\sqrt{\sum_{i = 1}^N 1} = v_{max} \sqrt{N},
\end{equation}
where the last step follows from the row-normalization condition of $\ba^*$. Hence together with \eqref{eq:a-a2} and \eqref{eq:v-v}, we have 
\begin{equation*}
	\|\ba^* - \widehat \ba\|_F \leq \frac{\varepsilon}{v_{min}} + \frac{\sqrt{N} v_{max}^4}{2v_{min}^2}\|\bEta^* - \widehat \bEta\|
\end{equation*}
and \eqref{eq:a-bound} follows from \eqref{eq:eta-bound} and above expression. 
\end{proof}

The errors in the recovered kernels and graph matrix scale linearly with $\varepsilon$ and $Q^{1/2}$, with constants critically determined by the disparity in kernel coefficients (i.e., relative kernel scales). The kernel error grows linearly with $N$, while the graph error grows with a higher rate of $N^{3/2}$, since the Frobenius norm of the graph matrix $\ba$ itself scales with $N^{1/2}$.

We now discuss the well-posedness of the linear system $\bA \bEta = \bf1_N$. As mentioned in Section \ref{sec:post_processing}, $\bA$ might not be invertible. In fact, we shall see that $\rank(\bA)$ is closely related to the number of active kernels. 

\begin{proposition}[Rank of $\mathbf{A}$ and number of active kernels]
Let $\bA$ be defined in \eqref{eq:A_eta} and recall $\calQ_{act}$ defined in \eqref{eq:Q_act}. Assume that $N \geq |\calQ_{act}|$, then 
\begin{equation}
	\rank(\bA) \leq |\calQ_{act}|. 
\end{equation}
\end{proposition}

\begin{proof}
Let $v_q:=\|\mathbf{c}^{(q)}\|>0$ for $q\in[Q]$. For each $i\in[N]$ and $q\in[Q]$ set 
\[
S_{iq}:=\sum_{j\neq i:\,\kappa^*_{ij}=q} a_{ij}^{2}
\quad\text{and}\quad 
\mathbf{S}:=\bigl[S_{iq}\bigr]\in\mathbb{R}^{N\times Q}.
\]
By definition of $\bKa_i$ in \eqref{Def:Assign_bKa}, 
\[
A_{iq}=\sum_{j\neq i} z_{ij}^2 \, (\bKa_i)_{jq}
=\sum_{j\neq i:\,\kappa_{ij}=q} z_{ij}^2.
\]
Using $z_{ij}=a_{ij}v_{\kappa_{ij}}$, we get
\[
A_{iq}=v_q^2 \sum_{j\neq i:\,\kappa_{ij}=q} a_{ij}^2
= v_q^2\, s_{iq},
\]
i.e. $\mathbf{A}=\mathbf{S}\,\mathrm{Diag}(v_1^2,\dots,v_Q^2)$. Since $v_q>0$, right-multiplication by the diagonal matrix does not change column dependencies on the active columns, so
$\rank(\mathbf{A})=\rank(\mathbf{S})$. Now suppose there exists $q\in[Q]$ such that $q \notin \calQ_{act}$, then it holds $S_{iq} = 0, i \in [N]$, i.e. the $i$-th row of $\bS$ are zeros, hence $\rank(\bA)  = \rank(\bS) \le |\calQ_{act}|.$
\end{proof}

\begin{remark}[Rank deficiency and kernel norm identifiability]\label{rmk:rank_A}
If $\operatorname{rank}(\mathbf{S}) = |\mathcal{Q}_{\mathrm{act}}| < Q$, then some kernels are \emph{inactive} and do not appear in the system; these can be ignored, and the estimation is performed only on the active kernels.

If $\operatorname{rank}(\mathbf{S}) < |\mathcal{Q}_{\mathrm{act}}|$, the active kernel columns of $\mathbf{S}$ are linearly dependent, preventing unique identification of kernel norms. This occurs, for instance, when two kernel types always appear in fixed proportions across all rows. Let $N=3$ and $Q = |\mathcal{Q}_{\mathrm{act}}|=2$. Set
\[
\mathbf{a} =
\begin{bmatrix}
0 & \sqrt{1/3} & \sqrt{2/3} \\
\sqrt{1/3} & 0 & \sqrt{2/3} \\
\sqrt{1/3} & \sqrt{2/3} & 0
\end{bmatrix},
\quad
\boldsymbol{\kappa} =
\begin{bmatrix}
* & 1 & 2 \\
1 & * & 2 \\
1 & 2 & *
\end{bmatrix}
\implies
\mathbf{S} =
\begin{bmatrix}
1/3 & 2/3 \\
1/3 & 2/3 \\
1/3 & 2/3
\end{bmatrix},
\]
whose columns are proportional, so $\mathrm{rank}(\mathbf{S}) = 1 < 2$. Since $\mathbf{A} = \mathbf{S}\,\mathrm{Diag}(v_1^2,v_2^2)$, this dependency carries over to $\mathbf{A}$, and the system \eqref{eq:LS_eta} admits infinitely many solutions for $\boldsymbol{\eta}$.

While a full column rank of $\mathbf{S}$ on $\mathcal{Q}_{\mathrm{act}}$ is necessary for unique kernel norm recovery, non-uniqueness here is not critical: the normalization constraint on $\mathbf{A}$ is somewhat artificial, and the essential interaction structure is encoded in $\mathbf{Z}$. One can still recover kernels accurately \emph{up to a constant scaling factor}, which suffices for structural identification.
\end{remark}

\section{Numerical Experiments}\label{Sec:Numerics}

In this section, we present numerical examples illustrating the performance of our algorithms, using the following error metrics.   For the matrix sensing step, we measure  
\begin{equation}
    \textbf{$\bZ$ error: } \|\bZ^* - \widehat{\bZ}\|_{2, \infty},
\end{equation}
where $\bZ^*$ and $\widehat{\bZ}$ denote the true and estimated kernel-graph embedding matrices, respectively. 
We use the absolute error so that this metric can be directly compared with the threshold $\varepsilon_0$ in~\eqref{eq:eps_0}. In the clustering step, we quantify misclassification using
\begin{equation}
	    \textbf{Type error: } 
    \frac{\#\{(i,j) : a^*_{ij}\neq 0, \kappa^*_{ij} \neq \widehat{\kappa}_{ij}\}}{N(N-1)},
\end{equation}
where $\kappa^*_{ij}$ is the true kernel type for pair $(i,j)$. Index pairs $(i,j)$ with $a^*_{ij} = 0$ are excluded since the corresponding kernel is inactive (see Remark \ref{rmk:act_inact_kernel}). To resolve label-permutation ambiguity, kernels are relabeled in the order they first appear in the system (see Remark \ref{rmk:cluster_order}). We evaluate the factorization step by measuring the recovery of the graph structure and interaction kernels. The relative graph error is 
\begin{equation}
    \textbf{Graph error: } 
    \frac{\|\ba^* - \widehat{\ba}\|_F}{\|\ba^*\|_F}, 
\end{equation}
where $\ba^*$ is the true graph matrix and $\|\ba^*\|_F = \sqrt{N}$ due to row-normalization. For the kernels, the relative error for type $q$ is 
\begin{equation}
    \textbf{Kernel error: } 
    \frac{\|\Phi^*_q - \widehat{\Phi}_q\|_{L^2(\rho)}}{\|\Phi^*_q\|_{L^2(\rho)}}, \quad q = 1, \dots, Q,
\end{equation}
where $\Phi_q^*$ is the $q$-th true kernel. When presenting convergence results, we report the average kernel error, computed as the mean of the $Q$ individual errors. 
Finally, we evaluate prediction accuracy by comparing trajectories generated from the estimated parameters (denoted by $\widehat{\bX}$) with those from the true parameters, using identical stochastic forcing and initial conditions. To mitigate the effect of potential type misclassification, the trajectories are generated directly from the estimated embedding matrix $\widehat{\bZ}$ rather than $(\widehat \ba, \{\widehat \Phi_q\}, \widehat\bkappa)$. The trajectory prediction error is defined as
\begin{equation}
    \textbf{Trajectory error: }
    \frac{1}{M'T} \sum_{m=1}^{M'} \|\bX_t^{(m)} - \widehat{\bX}_t^{(m)}\|_{L^2([0,T]; \,\R^{N \times d})}
    \ \approx\ 
    \frac{1}{M'L} \sum_{m=1}^{M'} \sum_{l=1}^{L} 
    \|\bX_{t_l}^{(m)} - \widehat{\bX}_{t_l}^{(m)}\|_{\R^{N\times d}}.
\end{equation}

In our numerical experiments, unless otherwise specified, we assume radial interaction kernels:
\begin{equation}\label{Def:radial_ker}
    \Phi_{\kappa_{ij}}(X^j_t - X^i_t) 
    = \phi_{\kappa_{ij}}\big(|X^j_t - X^i_t|\big) 
      \frac{X^j_t - X^i_t}{|X^j_t - X^i_t|},
    \quad X^i_t, X^j_t \in \R^d.
\end{equation}
This assumption simplifies the computations but is not essential; our method applies equally to general, non-radial kernels.

We start with a typical example of joint inference of multitype interaction kernels on the graph. Then we study the convergence of the estimation error with respect to the sample size $M$. 

\subsection{Typical example: predator-prey}

\begin{figure}[htbp]
\center
\includegraphics[width=1\textwidth]{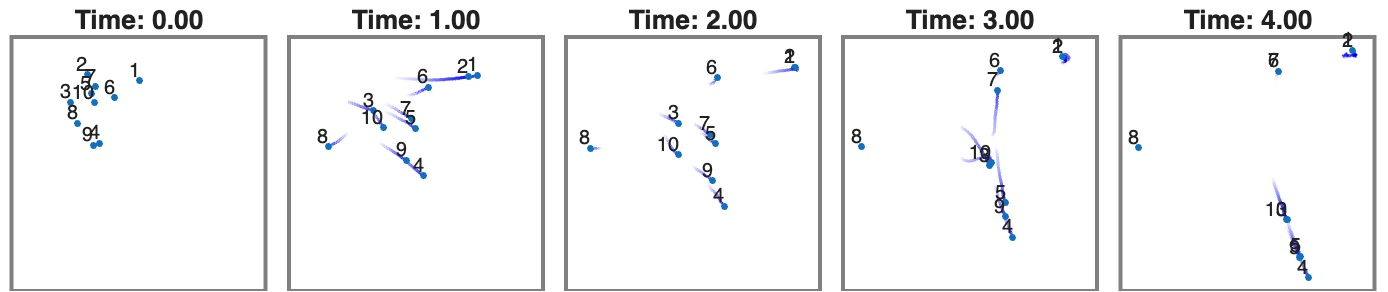}
\includegraphics[width=1\textwidth]{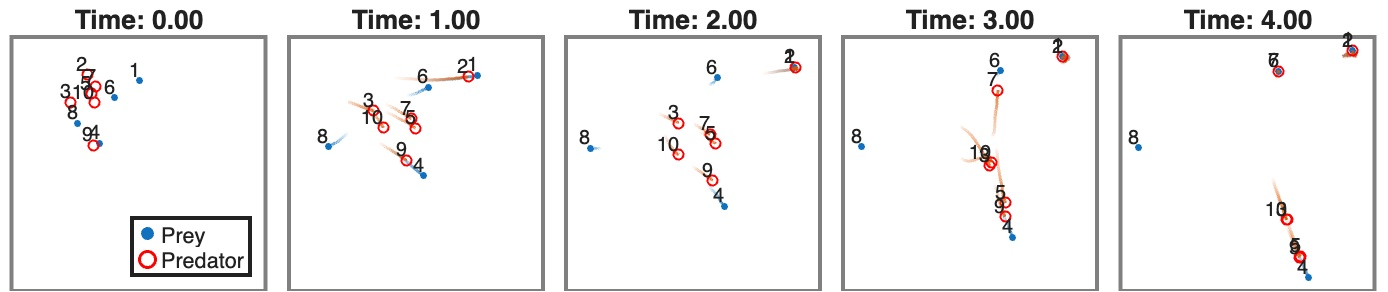}
\caption{Typical trajectory of a predator-prey system. \textbf{Top:} Observed trajectory data of 6 predators and 4 prey, with roles of individual particles unknown. \textbf{Bottom:} Ground-truth classification of predators and prey.}
\label{fig:traj}
\end{figure}

We first consider a predator-prey system with $N = 10$ particles, consisting of 6 predators and 4 prey. Four interaction types are present: predator-prey ($\phi_1$), prey-prey ($\phi_2$), prey-predator ($\phi_3$), and predator-predator ($\phi_4$). The kernels are represented using cubic splines on [0,5] with 10 knots, giving a hypothesis space with dimension $K = 13$. Positive kernel values indicate attraction, while negative values indicate repulsion. Predators exhibit long-range attraction toward prey, modeling sharp vision, whereas prey display only short-range repulsion from predators, reflecting lower alertness. Predator-predator interactions show mild attraction to facilitate cooperation, and prey-prey interactions exhibit mild repulsion to promote dispersion. See true kernels in Figure \ref{fig:typical_example_graph_kernel}, and a typical trajectory in Figure \ref{fig:traj}.

We generate $M = 20$ independent trajectories in $\mathbb{R}^2$, each with $L = 20$ time steps ($\Delta t = 0.1$), starting from uniformly sampled initial positions in $[0,4]^2$. 
The dynamics are integrated using the Euler method, with derivatives approximated by finite differences. 
A stochastic forcing with viscosity $\sigma = 10^{-3}$ and observation noise $\sigma_o = 10^{-3}$ are applied.  

The joint estimation errors are presented in Table \ref{table:typical_example}, and the results for the graph and kernels are shown in Figure~\ref{fig:typical_example_graph_kernel}. 
The type matrix is exactly recovered. In this example, $z_{min} = 0.5043$, $\theta_{min} = 0.5967$ giving $\varepsilon_0 = 0.0437$. Recall that condition~\eqref{eq:eps_0} is sufficient but not necessary for successful clustering. 
The graph error is small, and the kernel estimates closely follow the true curves, with only minor discrepancies for small exploration measure $\rho$, where data are scarce. 
As reported in the table, the post-processing ALS step further reduces all error metrics.

\begin{table}[htbp]
\center
\begin{tabular}{|c|c|c|c|cccc|c|}
\hline
\multirow{2}{*}{\textbf{Type of errors}} & \multirow{2}{*}{Z} & \multirow{2}{*}{Type} & \multirow{2}{*}{Graph} & \multicolumn{4}{c|}{Kernel}                                                                              & \multirow{2}{*}{Trajectory} \\ \cline{5-8}
                                &                    &                       &                        & \multicolumn{1}{c|}{$\phi_1$} & \multicolumn{1}{c|}{$\phi_2$} & \multicolumn{1}{c|}{$\phi_3$} & $\phi_4$ &                             \\ \hline
\textbf{Factorization}                      & 0.173           & 0                     & 0.036               & \multicolumn{1}{c|}{0.046} & \multicolumn{1}{c|}{0.298} & \multicolumn{1}{c|}{0.207} & 0.060 & 1.378                    \\ \hline
\textbf{Post-processing}             & 0.115           & 0                     & 0.033               & \multicolumn{1}{c|}{0.023} & \multicolumn{1}{c|}{0.106} & \multicolumn{1}{c|}{0.050} & 0.022 & 0.549                    \\ \hline
\end{tabular}
\caption{Error metrics for the predator-prey system.}
\label{table:typical_example}
\end{table}

\begin{figure}[ht]
\center
\includegraphics[width=0.51\textwidth]{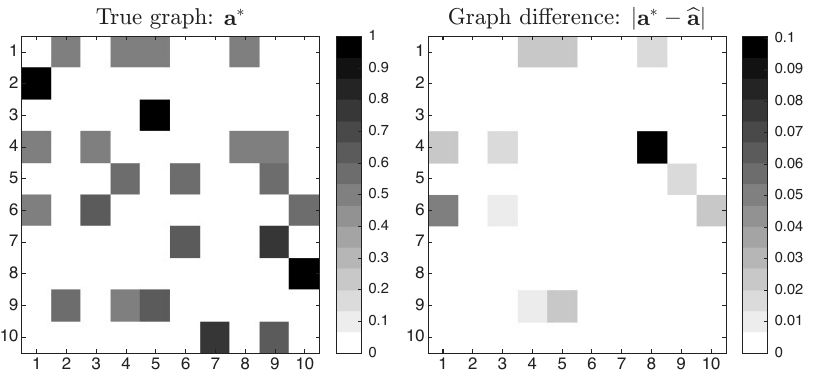}
\includegraphics[width=0.48\textwidth]{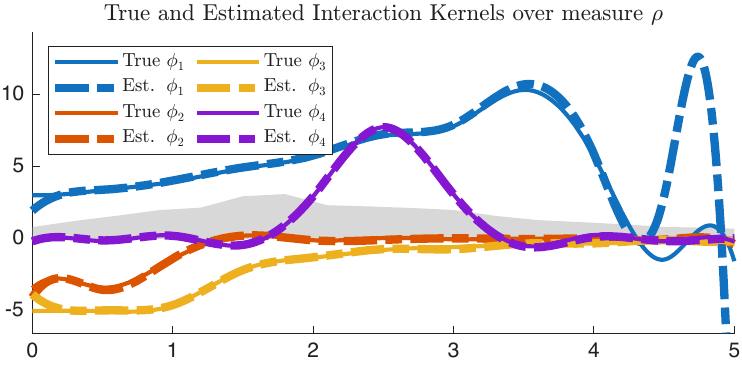}
\caption{Graph, kernel, and their estimations in the predator-prey system: predator-prey ($\phi_1$), prey-prey ($\phi_2$), prey-predator ($\phi_3$), and predator-predator ($\phi_4$). }
\label{fig:typical_example_graph_kernel}
\end{figure}

\subsection{Convergence with sample size $M$}\label{sec:conv_M}

We study the convergence of the estimation errors as the sample size $M$ increases.
We considered a system with $N = 12$ agents and interaction kernels generated from a Fourier basis with random coefficients, with hypothesis space dimension $K=10$, and the number of kernel types fixed at $Q=3$. 
The type matrix $\kappa$ is chosen randomly, and trajectories of length $L=1$ are simulated in dimension $d=2$ with time step $\Delta t = 10^{-1}$ and viscosity $\sigma = 10^{-4}$. The sample size $M$ varies log-uniformly between $2^4$ and $2^{10}$.
For each $M$, the experiment is repeated 10 times, and we report the mean error together with the 90\% quantile.
Errors in $\bZ$, the graph, kernels, and trajectory prediction are shown on a logarithmic scale.
The type classification error is reported as a ratio between 0 and 1, and whenever the number of clusters is not correctly recovered, we assign the kernel error a default value of $10^3$. We consider two graph structures: one relatively sparse and the other denser. 
The difference is governed by the minimum nonzero entry $a_0$ of the adjacency matrix $\ba$, 
set to $0.05$ for the dense graph and $0.25$ for the sparser one. 
Owing to the row-normalization constraint, a larger $a_0$ corresponds to increased sparsity. 
The chosen kernel settings yields $\theta_{\min} = 0.8272$ 
and a minimum kernel coefficient norm of $v_{\min} = 0.9273$. 
Since the threshold for successful recovery depends on $z_{\min}$, 
which itself is determined by $a_0$ and $v_{\min}$ via~\eqref{eq:z0_a0_c0}, 
this setup enables a direct comparison of how graph sparsity influences estimation performance. 
The corresponding results are reported in Figure~\ref{fig:conv_M}.

In the matrix sensing stage, Proposition~\ref{prop:consequence_RIP} establishes that, under the RIP assumption, the estimation error of $\bZ$ decays at rate $O(M^{-1/2})$. Empirically, for the sparse graph, we fit a decay rate of $0.58$, consistent with the theoretical prediction. In the clustering stage, Propositions~\ref{prop:support} and~\ref{prop:separate} guarantee that once $\|\bZ^* - \widehat{\bZ}\|_{2,\infty} < \varepsilon_0$, our algorithms ensure recovery of the type matrix $\bkappa$ in the sense of~\eqref{eq:kappa_recovery}. This behavior is confirmed by the close correspondence between the $\bZ$ error and the type error in our experiments. In our setting of two graphs, the corresponding separation thresholds are $\varepsilon_0 = 0.0147$ and $\varepsilon_0 = 0.0835$, so type recovery in the dense graph requires larger sample sizes. Finally, in the factorization stage, Theorem~\ref{thm:postproc-stability} implies that, once $\bkappa$ is exactly recovered, both the graph and kernel errors decay at the same rate as the $\bZ$ error, which is also observed numerically. Overall, the experiments corroborate the theoretical guarantees developed in Section~\ref{Sec:Theory}.

We emphasize that the trajectory error exhibits decay even without successful recovery of the type matrix, since the estimated trajectories are generated directly from the embedding matrix $\bZ$. This highlights that, in practice, accurate recovery of the graph and kernels may be less critical, particularly for real-world data where their interpretation can be ambiguous; whereas the low-rank embedding $\bZ$ fully characterizes the system dynamics. 
Furthermore, we observe that the decay behavior of the $\bZ$ error is identical in both sparse and dense graph settings. This is expected, as only one time step is applied, making the sensing matrices in the matrix sensing stage identical. This observation further agrees with our theoretical guarantee for the first stage of matrix sensing.

\begin{figure}[t]
\center
\includegraphics[width=0.48\textwidth]{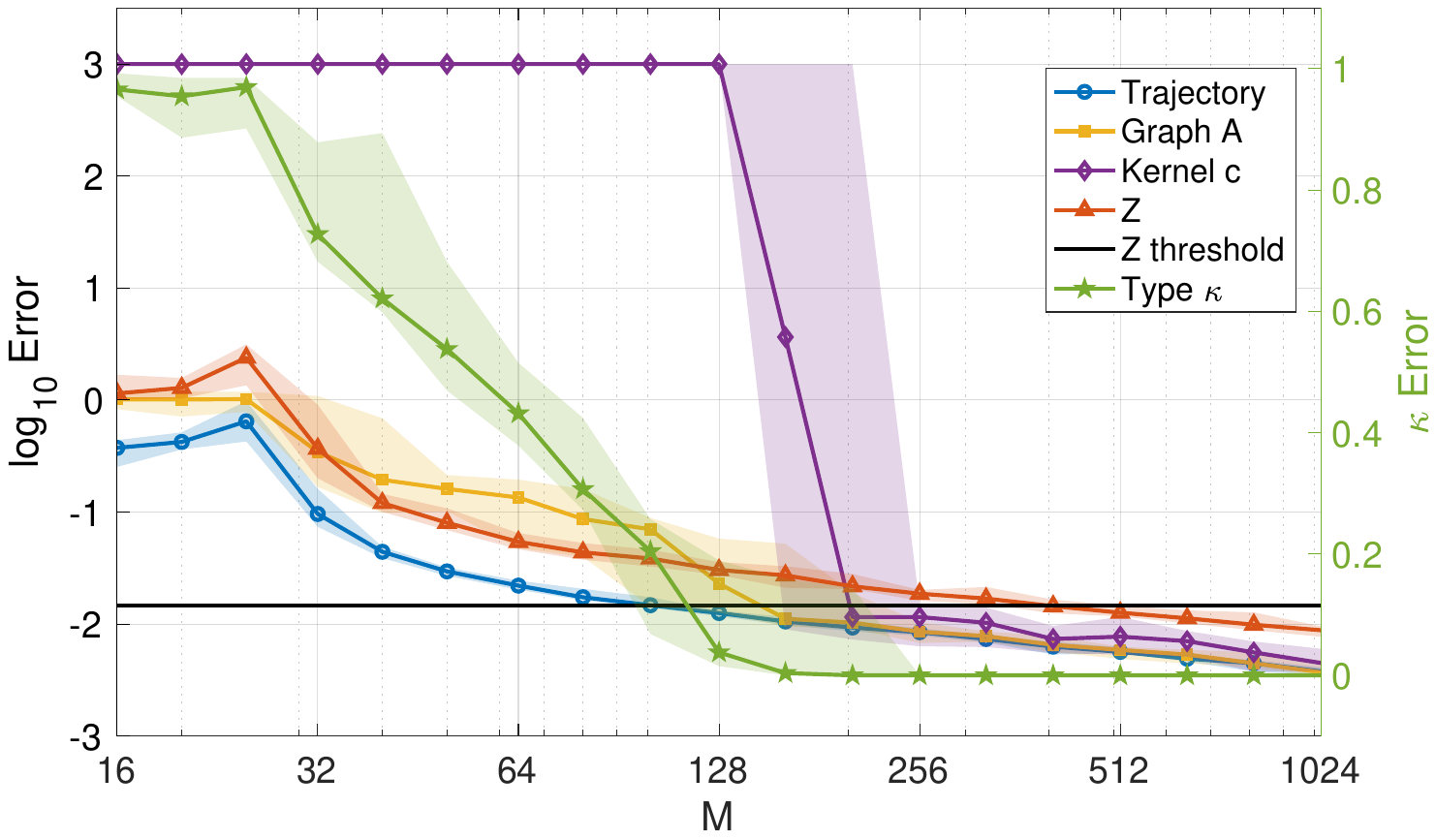}
\includegraphics[width=0.48\textwidth]{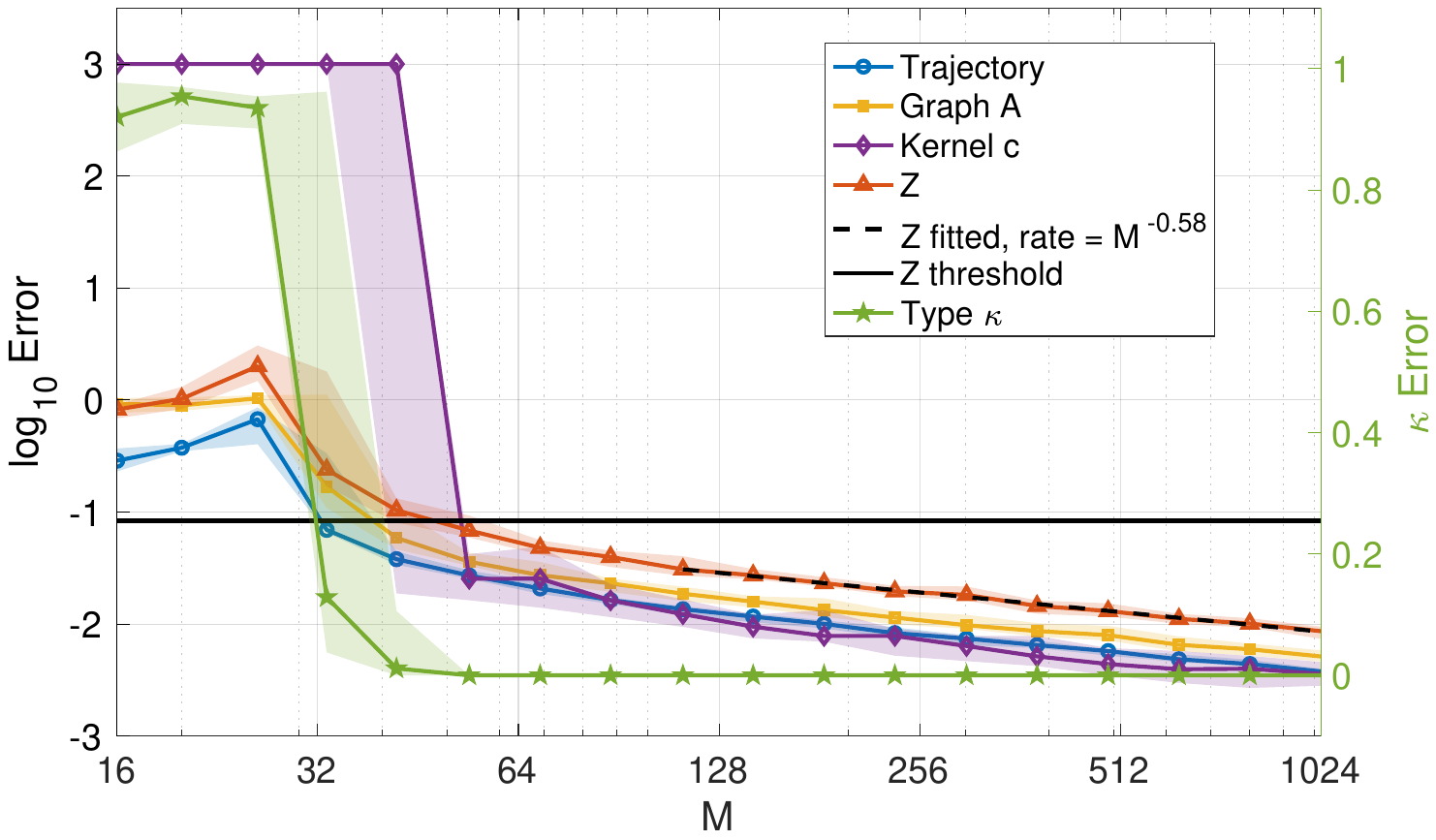}
\caption{{Convergence of estimation errors with sample size $M$.} \textbf{Left:} dense graph. \textbf{Right:} sparse graph. 
Errors in $\bZ$ and trajectory prediction decrease once $M > 32$, with the $\bZ$ error decaying at the rate $M^{-1/2}$. Successful recovery of the type matrix $\kappa$ occurs when the $\bZ$ error falls below the separation thresholds (around $M=300$ and $M = 54$), after which kernel errors also begin to decay. 
Sparse graphs require fewer samples for recovery than dense graphs.
}
\label{fig:conv_M}
\end{figure}

\subsection{Convergence with number of agents $N$}\label{sec:conv_N}
We study the convergence of the estimation error as $N$ increases from 4 and 64, using the same system settings as in the as in the previous example. That is, we consider $Q = 3$ kernels from a random Fourier basis with $K = 10$, and generate trajectories in $d = 2$ with time step $\Delta t = 0.1$ and a stochastic forcing with viscosity $\sigma = 10^{-4}$. We take $L = 10$, $M = 100$, which are sufficiently large so that the clustering estimation is exact and the type error vanishes. 

In particular, we use a network connectivity parameter $a_0 = 0.25$, which makes the graph sparser as $N$ increases, since it admits at most 16 non-zero entries per row due to row normalization.  This setting yields the same $\varepsilon_0$ as in the previous example, $\varepsilon_0 = 0.0835$. To highlight the effect of graph sparsity, we also consider a known complete graph setting, where all agents interact with each other with equal strength.

\begin{figure}[t]
\center
\includegraphics[width=\textwidth]{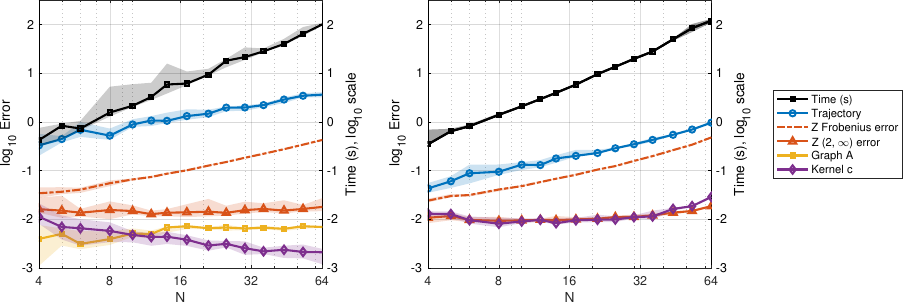}
\caption{Scaling of errors with respect to $N$. \textbf{Left:} Unknown sparse graph.  The trajectory error increases approximately linearly in $N$, whereas the computation time increases quadratically. The Frobenius $\bZ$ error increases with a rate close to the theory, while the $(2,\infty)$ error of $\bZ$ grows much slower. The graph error remains small for the same reason, and the kernel error decreases as the matrix $\bA$ becomes better conditioned. \textbf{Right:} Known complete graph. The trajectory error, computation time, and $Z$ Frobenius error behave similarly to the case of an unknown graph. However, the kernel error increases as $N$ increases, since the graph with increasing size impedes the estimation of $Z$. 
} 
\label{fig:conv_N}
\end{figure}

The left plot in Figure \ref{fig:conv_N} presents the scaling of errors and computation time with respect to $N$ when the graph is unknown and sparse. The computation time scales quadratically with $N$, consistent with the analysis in Section~\ref{subsec:K-means_known}. The trajectory error scales as $N^{0.77}$, reflecting the increase in the number of agents; the rate is slightly sublinear due to graph sparsity. For the matrix errors, the Frobenius error of $\bZ$ scales linearly in agreement with \eqref{eq:Z_error}, while the $(2,\infty)$-error scales only as $N^{0.08}$. The latter reflects the sparsity of the interaction graph: as $N$ increases with a fixed lower bound $a_0$, the graph becomes progressively sparser, leading to slower growth. Notably, the dependence on $M$ in \eqref{eq:Z_error} and \eqref{eq:Z_error_2_infty} is the same, but the dependence on $N$ differs, which is precisely borne out by this example.  As a consequence of the small $(2,\infty)$-error, the normalized graph error scales as $N^{0.20}$, which is smaller compared to the linear order predicted by \eqref{eq:a_c_error_approx}, again due to sparsity. The kernel error actually decays as $N^{-0.59}$. This decay is explained by the improved conditioning of $\bA$ as $N$ grows: the smallest eigenvalue increases, reducing the error. Intuitively, as noted in Remark~\ref{rmk:rank_A}, rank deficiency, and therefore ill-posedness, occurs only under strong symmetries in $\bA$, which are unlikely in randomly generated graphs, especially when $N$ is large. Nevertheless, the growing rate is controlled by the theoretical rate provided in Theorem \ref{thm:postproc-stability}.

The right plot of Figure \ref{fig:conv_N} presents results when the graph is known and \emph{complete}, using the algorithm in Appendix \ref{apdx:complete_networks}. The trajectory error, computation time, and Frobenius error of $\bZ$ behave similarly to the unknown graph case. However, the kernel error increases as $N$ increases. This is because, even though the graph is known, its increasing size impedes the estimation of $\bZ$, leading to larger kernel errors. Thus, the benefit of knowing the graph is limited when it is dense. 

\subsection{Necessary sample size}\label{sec:sample-size}
We now demonstrate the necessary sample size $M$ to successfully recover the type assignment $\kappa$ when the number of particles $N$ and the dimension of the hypothesis space $K$ grow.

We take $N$ to range from 4 to 16 and $K$ from 2 to 10, and set $Q=2$. The true kernels are randomly generated Fourier functions and are approximated by $K$ cubic spline basis functions. This ensures the structural prior information for $z_{min}$ and $\theta_{min}$ does not change much. We fix the graph to be complete and use the algorithm in Appendix \ref{apdx:complete_networks}, as this setting also allows for stable structural priors. For each pair of $N$ and $K$, we generate $M$ trajectories in $d = 2$ with $L = 1$ time steps. We range $M$ from 2 to 22, and for each fixed $M$, we run the experiment for $B = 3$ times. We record the minimal $M$ such that we have a successful recovery, in the sense that $\kappa$ is correctly recovered. We present the growth rate of the minimal necessary sample size $M$ with the growth of $N + K$ in Figure \ref{fig:M_NK}. It is clear that $M$ depends roughly linearly on $N+K$, which demonstrates the validity of Assumption \ref{conj_RIP}.

\begin{figure}[h!]
\center
\includegraphics[width=0.6\textwidth]{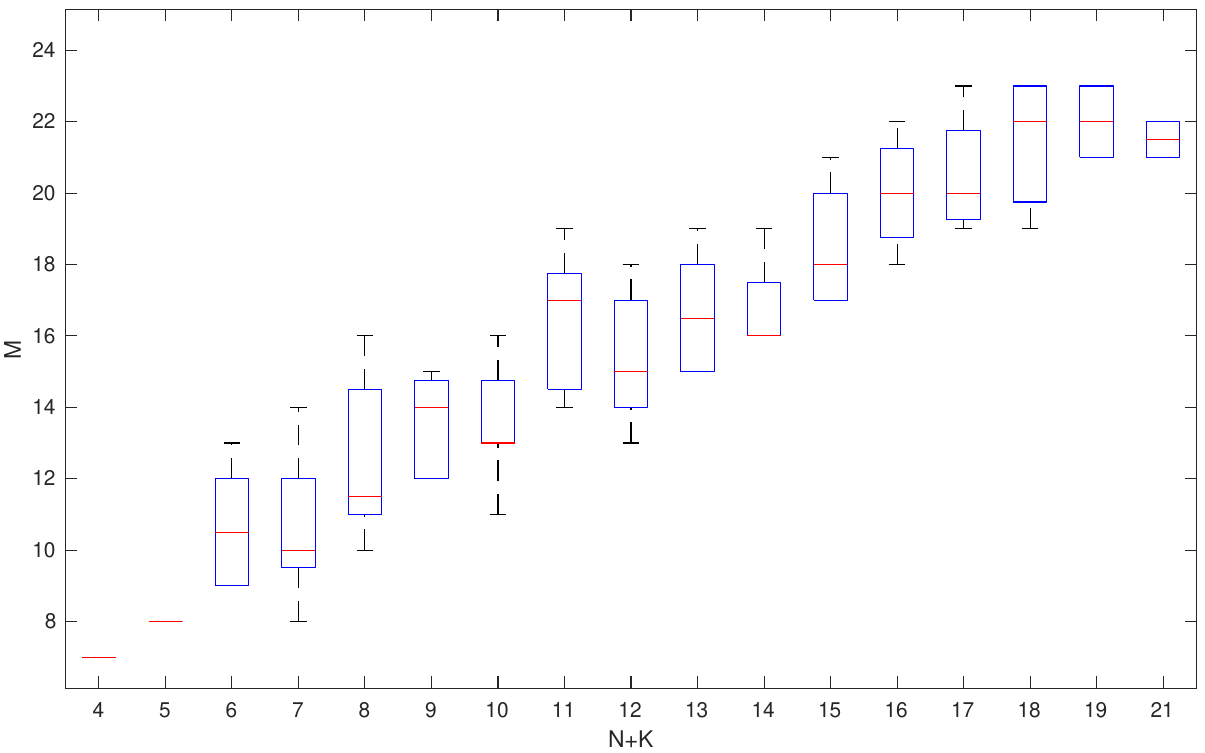}
\caption{Scaling of necessary sample size $M$ with $N+K$ to recover the type function $\kappa$, using the estimated $\bZ$ in the matrix sensing stage. }
\label{fig:M_NK}
\end{figure}

\section{Conclusion}\label{Sec:Conclusion}

We developed a unified framework for learning heterogeneous interacting particle systems (IPS) on networks, where both the network topology and the interaction kernels are unknown, and agents interact through latent types. By reformulating the joint estimation problem into a three-stage procedure, low-rank matrix sensing followed by geometric clustering and matrix factorization, we decoupled the continuous and discrete aspects of the optimization. This reformulation yields an efficient and robust algorithm with strong theoretical guarantees: RIP-based recovery bounds for the matrix sensing stage and exact type recovery under separation conditions for the clustering stage.

Our analysis demonstrates that the framework recovers the network structure, interaction kernels, and type assignments with provable accuracy, even when the number of types is unknown. Numerical experiments confirm the sharp recovery thresholds predicted by theory and demonstrate the approach's scalability.

Beyond the specific examples studied here, the proposed methodology applies broadly to heterogeneous IPS across biology, physics, and social sciences. Extensions to more complex models, including higher-order dynamics, incomplete observations, or time-varying networks, are natural directions for future work. Overall, this work provides a principled and scalable pathway toward data-driven discovery of structured interactions in multi-agent systems.

\appendix

\section{Complete Networks}\label{apdx:complete_networks}

\subsection{Algorithms}
We present the details of the procedure for learning the type matrix $\bkappa$ and the coefficient matrix $\bc$ in the case of a complete interaction graph, as discussed in Remark~\ref{rmk:complete_network_2} and Section \ref{sec:complete_network}. The first step involves estimating the embedding matrix $\bZ$ via Algorithm~\ref{alg:ALS}. In the subsequent stage, if the number of kernels $Q$ is known, we apply positional clustering to the rows of $\bZ$, as described in Algorithm~\ref{alg:Kmeans_known_Q_complete_network}. When $Q$ is unknown, we determine it by comparing the quantities $\theta$ and $\Theta$ against appropriate thresholds, following the procedure in Algorithm~\ref{alg:Kmeans_unknown_Q_complete_network}. In both scenarios, post-processing is not required.

\begin{algorithm}[htbp]
\caption{Clustering $\widehat{\bZ}$ with known $Q$ for complete networks}
\label{alg:Kmeans_known_Q_complete_network}
\begin{algorithmic}[1] 
\Procedure{Positional K-means, $Q$ known}{estimated $\widehat{\bZ}$, $z_{0}$, $Q$}
	\State Use K-means to classify the rows of $\widehat\bZ$ with into $Q$ clusters with Eucledian distance, get index sets $\{I_q\}_{q = 1}^Q$ and cluster centers $\{\mathbf{v}_q\}_{q = 1}^Q$.  
	\State Construct $\widehat{\bkappa} \in [Q]^{N\times N}$ by setting $\widehat{\kappa}_{ij} = q$ if $(i, j) \in I_q$ for $q = 1, \dots, Q$. 
	\State Construct $\widehat \bc = [\bv_1, \dots, \bv_Q]$. 
	\State \textbf{return} $\widehat{\bkappa}, \widehat{\bc}$. 
    \EndProcedure
\end{algorithmic}
\end{algorithm}

\begin{algorithm}[htbp]
\caption{Clustering $\widehat{\bZ}$ with unknown $\widehat{Q}$ for complete networks}
\label{alg:Kmeans_unknown_Q_complete_network}
\begin{algorithmic}[1]
\Procedure{Positional K-means, $Q$ unknown}{$\widehat{Z}$, $\theta_0$, $\Theta_{max}$,  ${Q}_0$, $n_{\max}$ }
    \State Initialize $\widehat{Q} = {Q}_0$, \texttt{success} = \texttt{false}
    
    \For{$\tau = 0, \dots, n_{\max}$}
        \State Run K-means with $\widehat{Q}$ clusters on rows of ${\widehat \bZ}$ using Eucledian distance, get index sets $\{I_q\}_{q = 1}^{\widehat{Q}}$ and cluster centers $\{\mathbf{v}_q\}_{q = 1}^{\widehat{Q}}$.  
        \State Compute min inter-cluster angle $\theta$ and max intra-cluster angle $\Theta$
        
        \If{$\theta > \theta_0$ and $\Theta < \Theta_{\max}$}
            \State \texttt{success} = \texttt{true}, \textbf{break}
        \ElsIf{$\theta < \theta_0$}
            \State $\widehat{Q} \gets \widehat{Q} - 1$
        \ElsIf{$\Theta > \Theta_{\max}$}
            \State $\widehat{Q} \gets \widehat{Q} + 1$
        \EndIf
    \EndFor

    \If{\texttt{success} = \texttt{false}}
        \State \textbf{return} Classification fails
    \Else
	\State Construct $\widehat{\bkappa} \in [Q]^{N\times N}$ by setting $\widehat{\kappa}_{ij} = q$ if $(i, j) \in I_q$ for $q = 1, \dots, Q$. 
	\State Construct $\widehat \bc = [\bv_1, \dots, \bv_Q]$. 
	\State \textbf{return} $\widehat{\bkappa}, \widehat{\bc}$. 
    \EndIf
\EndProcedure
\end{algorithmic}
\end{algorithm}

\subsection{Numerical Experiments}
We present the convergence results for the case of a complete graph. All experimental settings are the same as in Section~\ref{sec:conv_M}, except that the network is taken to be complete and estimation is performed using Algorithm~\ref{alg:Kmeans_known_Q_complete_network}. The corresponding errors are shown in Figure~\ref{fig:conv_M_known_A}.
\begin{figure}[t]
\center
\includegraphics[width=0.7\textwidth]{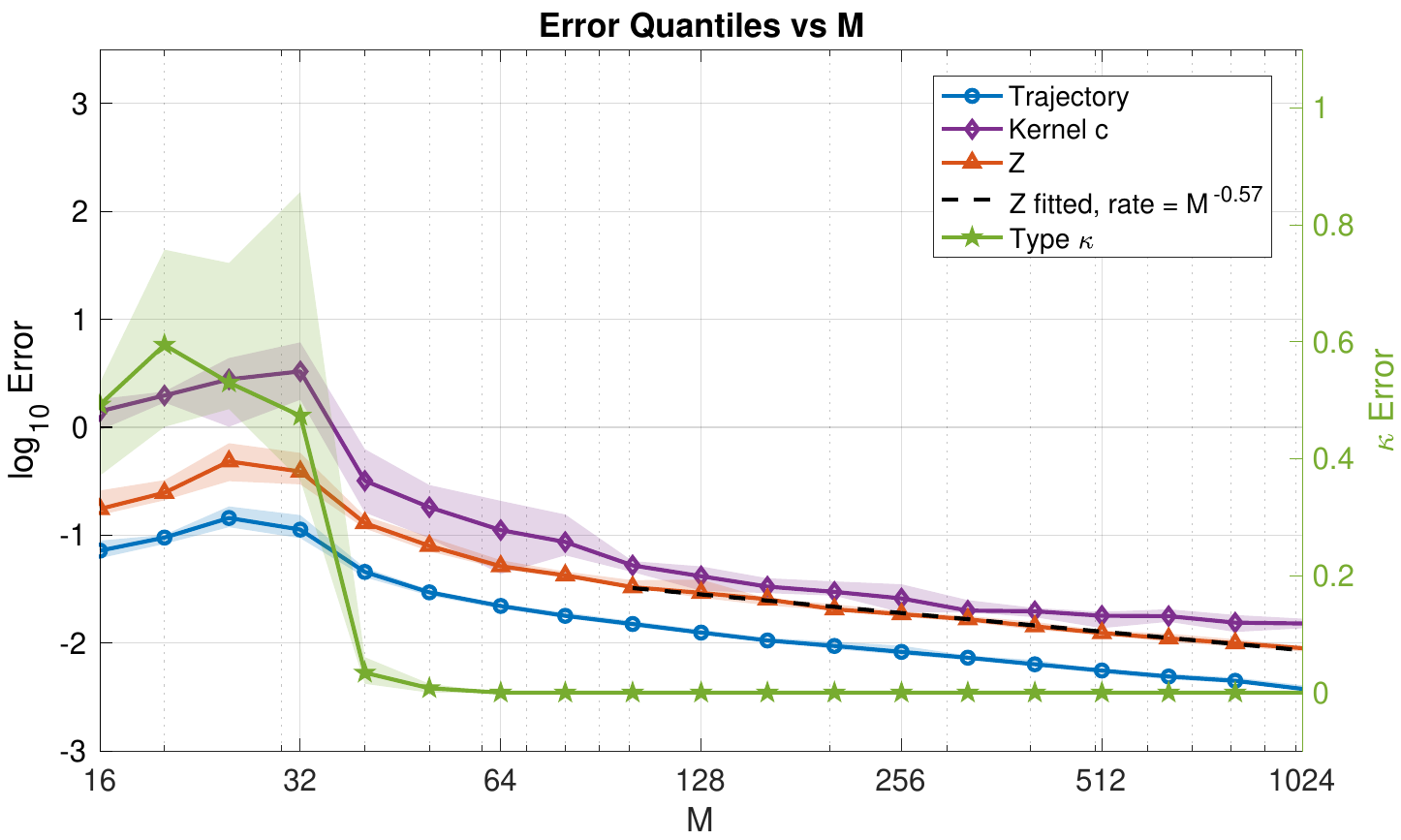}
\caption{Convergence of estimation errors with sample size $M$ when the graph is known. In this setting, graph recovery is not required, so we report only the kernel error and type error. Because all kernels are active and the number of kernel types $Q$ is known, kernel estimation can be evaluated even with very few trajectories. The results show that the kernel error decreases with $M$ at essentially the same rate as the $\bZ$ error. A fitted decay rate of $-0.57$ confirms this behavior, consistent with the theoretical matrix sensing rate. Compared to the general unknown-graph case, substantially fewer samples are needed.}
\label{fig:conv_M_known_A}
\end{figure}

\section{Proofs}
\subsection{Preliminary on matrix sensing and RIP}\label{sec:preRIP}
The matrix sensing problem aims to find a low-rank matrix $Z^*\in \R^{n_1\times n_2}$ from data $b_m=\langle A_m,Z^*\rangle_{F}$, where $A_1,\cdots,A_M\in \R^{n_1\times n_2}$ are sensing matrices$\footnote{We set $A_m := \bB_i(\bX^m_0) \in \R^{(N-1)\times K}$ in Section \ref{sec:RIP}. }$. 
To find $Z^*$ with rank $r\ll n_1 \wedge n_2$, one solves the following non-convex optimization problem
\begin{equation}\label{eq:matrix_sensing}
	\min_{Z\in\R^{n_1\times n_2},\text{rank}(Z)=r} F(Z) = \frac{1}{M}\sum_{m=1}^M |\langle A_m,Z\rangle_{F}-b_m|^2=\frac{1}{M}\sum_{m=1}^M |\Tr(A_m^\top Z)-b_m|^2\,.
\end{equation}
It is well-known that the constrained optimization problem \eqref{eq:matrix_sensing} is NP-hard. 
A common method of factorization is introduced by Burer and Monteiro \cite{BM2003,BM2005}. Namely, we express $Z=UV^\top$ where $U\in\R^{n_1\times r}$ and $V\in\R^{n_2\times r}$.
 Then \eqref{eq:matrix_sensing} can be transformed to an unconstraint problem
\begin{align}\label{eq:matrix_factor}
	\min_{U\in\R^{n_1\times r},V\in\R^{n_2\times r}} F(U,V) &= \frac{1}{M}\sum_{m=1}^M |\langle A_m,UV^\top\rangle_{F}-b_m|^2.
\end{align}
To simplify the notations, let us define a linear sensing operator $\sA :\R^{n_1\times n_2}\to \R^M$ by 
\begin{equation}\label{Def:SensOp}
	\sA(Z)= \bigg(\frac{1}{\sqrt{M}}\langle A_1,Z\rangle_{F},\cdots,\frac{1}{\sqrt{M}}\langle A_M,Z\rangle_{F}\bigg)\,.
\end{equation}
 
\begin{definition}[Restricted isometry property (RIP)]\label{Def:RIP} 
The linear map $\sA$  satisfy the $(r,\delta_{r})$-RIP condition with the RIP constant $\delta_{r}=\delta_r(\sA)\in[0,1)$ if there is a positive constant $C$ such that
\begin{equation}\label{eq:RIP}
	(1-\delta_{r}) \norm{Z}_F^2 \leq \frac{1}{C} \norm{\sA (Z)}_{\R^m}^2=\frac{1}{CM}\sum_{m=1}^M \langle A_m,Z\rangle_{F}^2 \leq (1+\delta_{r}) \norm{Z}_F^2
\end{equation}
holds for all $Z$ with rank at most $r$, and $\delta_r$ is smallest constant for the above inequalities. 
\end{definition}

The normalizing constant $C$ in Condition \eqref{eq:RIP}, introduced in \cite{RohdeTsybakov2011}, enables the application of RIP to a large class of sensing matrices that can be scaled to near isometry. In particular, in our setting, where the sensing matrices are generally far from an isometry, such scaling is particularly important.

The restricted isometry property plays a critical role in the theory of matrix sensing \cite{RFP2010}, a generalization of compressed sensing \cite{CandesTao2005}. It ensures the identifiability of matrix sensing problems when the data $\{b_m\}$ are noiseless. 
\begin{theorem}[Theorem 3.2 in \cite{RFP2010}]\label{Thm:RFP2010}
	Suppose that $\sA$ satisfies the $(2r,\delta_{2r})$-RIP condition with $\delta_{2r}<1$ and an integer $r\geq 1$. Then $Z^*$ is the only matrix of rank at most $r$ satisfying $\sA(Z)=b=[b_1,\cdots,b_M]^\top$ with $b_m=\langle A_m,Z^*\rangle_{F} $.
\end{theorem}

When the sensing data $\{b_m\}$ are noisy, a widely used regularizer for asymmetric problems $n_1 \neq n_2$ (see, e.g., \cite{zheng2016convergence, GJZ2017}) is  
\begin{equation}\label{eq:matrix_factor_reg}
	\min_{U\in\R^{n_1\times r},V\in\R^{n_2\times r}} F(U,V) = \frac{1}{M}\sum_{m=1}^M |\langle A_m,UV^\top\rangle_{F}-b_m|^2 + \frac{1}{4} \norm{U^\top U - V^\top V}_F^2, 
\end{equation} 
The next theorem provides an error bound for the regularized estimator. 

\begin{theorem}[\protect{Theorem 32 in \cite{GJZ2017}}]\label{thm:noisy_matrix_sensing}
	Suppose the sensing operator $\sA$ satisfies the $(2r, 1/20)$-RIP condition, and the observations are given by $b_m = \innerp{A_m, Z^*}_F + \varepsilon_m, m = 1, \dots M$, where the noise term $\{\varepsilon_m\}$ are i.i.d. Gaussian random vartiables with distribution $\mN(0, \sigma_\epsilon^2)$. Then, with high probability, every local minimum $(U, V)$ of the nonconvex objective \eqref{eq:matrix_factor_reg} satisfies 
	\begin{equation}
		\norm{UV^\top - Z^*}_F \lesssim  \sigma_\epsilon \sqrt{\frac{(n_1+ n_2)r\log M}{M}}.
	\end{equation}
	Here, a point $x=(U, V)$ is said to be a local minimum of $F$ if $\nabla F(x) = 0$ and $\nabla ^2 F(x) \succeq 0$. 
\end{theorem}

\subsection{Geometry Facts }\label{apdx:geom}

The following geometry facts are useful for providing the separability in Section \ref{sec:separation}. 
\begin{lemma}\label{lem:geom}
Let $ x, y \in \mathbb{R}^d $ be nonzero vectors , and define their normalized versions as $ \bar{x} = \frac{x}{\|x\|} $, $ \bar{y} = \frac{y}{\|y\|} $. Then, 
\[
 \|\bar{x} - \bar{y}\|  \leq \frac{1}{\min(\|x\|, \|y\|)} \|x - y\|\]
 and  
 \begin{equation}\label{eq:trig_id}
 	\arccos\innerp{\bar{x}, \bar{y}} = 2\arcsin\frac{\|\bar{x}-\bar{y}\|}{2} \leq \frac{\pi}{2}\|\bar{x} - \bar{y}\|.
 \end{equation}
 \end{lemma}

\begin{proof}[Proof of Lemma \ref{lem:geom}] Without loss of generality, assume $ \|x\| \leq \|y\| $. We aim to show $
\|\bar{x} - \bar{y}\| \leq \frac{1}{\|x\|} \|x - y\|$. 
Squaring both sides and subtracting the left-hand side from the right-hand side, it suffices to prove
\[
 \frac{1}{\|x\|^2} \|x - y\|^2 - \|\bar{x} - \bar{y}\|^2 \geq 0. 
\]
Since $\|\bar{x} - \bar{y}\|^2 = 2 - 2\frac{\langle x, y \rangle}{\|x\| \|y\|}$ and $
\frac{1}{\|x\|^2} \|x - y\|^2 = 1 + \frac{\|y\|^2}{\|x\|^2} - 2 \frac{\langle x, y \rangle}{\|x\|^2}$, the above inequality is equivalent to that
\begin{equation} \label{eq:geom_diff}
\frac{\|y\|^2}{\|x\|^2} - 1 + 2 \frac{\langle x, y \rangle}{\|x\|} \left( \frac{1}{\|y\|} - \frac{1}{\|x\|} \right)\geq 0. 
\end{equation}

 By the Cauchy--Schwarz inequality $ |\langle x, y \rangle| \leq \|x\| \|y\| $, we have
\[
\left| 2 \frac{\langle x, y \rangle}{\|x\|} \left( \frac{1}{\|y\|} - \frac{1}{\|x\|} \right) \right| \leq 2 \|y\| \left( \frac{1}{\|x\|} - \frac{1}{\|y\|} \right) = \frac{2(\|y\| - \|x\|)}{\|x\|}.
\]
Meanwhile, since $ \|y\| \geq \|x\| $, we have 
\[
\frac{\|y\|^2}{\|x\|^2} - 1 = \frac{(\|y\| + \|x\|) (\|y\| - \|x\|)}{\|x\|^2} \geq \frac{2(\|y\| - \|x\|)}{\|x\|}.
\]
Hence, \eqref{eq:geom_diff} holds, with equality only if $ x $ and $ y $ are colinear and $ \|x\| = \|y\| $. 

To prove \eqref{eq:trig_id}, let $\theta=\arccos\langle \bar{x},\bar{y}\rangle\in[0,\pi]$.
Since $\bar{x}$ and $\bar{y}$ are unit vectors, we have 
$$
\|\bar{x}-\bar{y}\|^2=\|\bar{x}\|^2+\|\bar{y}\|^2-2 \langle \bar{x},\bar{y} \rangle
=2-2\cos\theta = 4\sin^2 \frac{\theta}{2},
$$
i.e., $\|\bar{x}-\bar{y}\|=2\sin(\theta/2)$. Thus, 
$$
2\arcsin \left(\frac{\|\bar{x}-\bar{y}\|}{2}\right)
=2\arcsin(\sin(\theta/2)) =\theta
=\arccos \langle \bar{x},\bar{y} \rangle,
$$
where we used that $\theta/2\in[0,\pi/2]$ so $\arcsin(\sin(\theta/2))=\theta/2$. The last inequality follows from that $\arcsin(x)\leq \frac{\pi x}{2}$ for all $x \in [0, 1]$. 
\end{proof}

\begin{proof}[Proof of Lemma \ref{lem:geom_avg}]
Since $\|\mathbf{z}_i\| = \|\mathbf{v}^\star\| = 1$, we have $\|\mathbf{z}_i - \mathbf{v}^\star\|^2 = 2 - 2\, \mathbf{z}_i^\top \mathbf{v}^\star$, hence
\[
\mathbf{z}_i^\top \mathbf{v}^\star = 1 - \tfrac{1}{2} \|\mathbf{z}_i - \mathbf{v}^\star\|^2 \;\ge\; 1 - \frac{\varepsilon^2}{2}.
\]
Averaging over $i$ gives
\[
\mathbf{v}^{\star\!\top} \mathbf{u} = \frac{1}{n} \sum_{i=1}^n \mathbf{z}_i^\top \mathbf{v}^\star \;\ge\; 1 - \frac{\varepsilon^2}{2}.
\]
If $\varepsilon < \sqrt{2}$ then $\mathbf{v}^{\star\!\top} \mathbf{u} > 0$, so $\mathbf{u} \neq \mathbf{0}$ and $\mathbf{v}$ is well-defined. Using $\|\mathbf{u}\| \le 1$,
\[
\mathbf{v}^{\star\!\top} \mathbf{v} = \frac{\mathbf{v}^{\star\!\top} \mathbf{u}}{\|\mathbf{u}\|} \;\ge\; \mathbf{v}^{\star\!\top} \mathbf{u} \;\ge\; 1 - \frac{\varepsilon^2}{2}.
\]
Finally, 
\[
\|\mathbf{v} - \mathbf{v}^\star\|^2 = 2\bigl(1 - \mathbf{v}^{\star\!\top} \mathbf{v}\bigr)
\;\le\; 2\left(1 - \Bigl(1 - \frac{\varepsilon^2}{2}\Bigr)\right) = \varepsilon^2,
\]
which yields $\|\mathbf{v} - \mathbf{v}^\star\| \le \varepsilon$.
\end{proof}

\section*{Acknowledgments}
This work was supported by National Science Foundation Grants DMS-2238486 and DMS-2511283; Air Force Office of Scientific Research Grant AFOSR-FA9550-20-1-0288, FA9550-21-1-0317, and FA9550-23-1-0445. 
\printbibliography

\end{document}